\documentclass[journal]{IEEEtran}
\usepackage{}

\ifCLASSINFOpdf
\else
\fi
\usepackage{array}

\usepackage[cmex10]{amsmath}
\usepackage{bm}
\usepackage{enumerate}
\usepackage{amsfonts}
\usepackage{graphicx}
\usepackage{cite}
\usepackage{booktabs}
\usepackage{multirow}
\usepackage[linewidth=0.5pt]{mdframed}
\usepackage{threeparttable}
\usepackage{fixltx2e}
\usepackage{color}
\usepackage{xcolor}
\usepackage{subeqnarray}
\usepackage{cases}
\usepackage{makecell}
\usepackage{algorithmic}
\usepackage{algorithm}
\usepackage{bbold}
\usepackage{amssymb}
\usepackage{epstopdf}
\usepackage{bm}
\usepackage{amsthm}
\theoremstyle{definition}
\newtheorem{theorem}{\textbf{Theorem}}
\newtheorem{definition}[theorem]{\textbf{Definition}}

\newtheorem{proposition}[theorem]{\textbf{Proposition}}
\newtheorem{remark}[theorem]{\textbf{Remark}}
\newtheorem{lemma}[theorem]{\textbf{Lemma}}

\newenvironment{shrinkeq}[1]
{ \bgroup
  \addtolength\abovedisplayshortskip{#1}
  \addtolength\abovedisplayskip{#1}
  \addtolength\belowdisplayshortskip{#1}
  \addtolength\belowdisplayskip{#1}}
{\egroup\ignorespacesafterend}
         
\ifCLASSOPTIONcompsoc
\usepackage[caption=false,font=normalsize,labelfon
t=sf,textfont=sf]{subfig}
\else
\usepackage[caption=false,font=footnotesize]{subfig}

\fi

\begin{document}
\bibliographystyle{IEEEtran}


\title{Bounding Regression Errors in \\Data-driven Power Grid Steady-state Models}
\author{
Yuxiao~Liu,~\IEEEmembership{Student Member,~IEEE,}
Bolun~Xu,~\IEEEmembership{Member,~IEEE,}
Audun~Botterud,~\IEEEmembership{Member,~IEEE,}
Ning~Zhang,~\IEEEmembership{Senior Member,~IEEE,}
and~Chongqing~Kang,~\IEEEmembership{Fellow,~IEEE}

	

}

\maketitle


\begin{abstract}
Data-driven models analyze power grids under incomplete physical information, and their accuracy has been mostly validated empirically using certain training and testing datasets. This paper explores error bounds for data-driven models under all possible training and testing scenarios drawn from an underlying distribution, and proposes an evaluation implementation based on Rademacher complexity theory. We answer critical questions for data-driven models: how much training data is required to guarantee a certain error bound, and how partial physical knowledge can be utilized to reduce the required amount of data. 
Different from traditional Rademacher complexity that mainly addresses classification problems, our method focuses on regression problems and can provide a tighter bound.
Our results are crucial for the evaluation and application of data-driven models in power grid analysis. We demonstrate the proposed method by finding generalization error bounds for two applications, i.e., branch flow linearization and external network equivalent under different degrees of physical knowledge.
Results identify how the bounds decrease with additional power grid physical knowledge or more training data.


\end{abstract}

\IEEEpeerreviewmaketitle


\section{Introduction}
Data-driven models are widely applied in power systems to model the system stability, power grid control and optimization strategies, electrical consumer behaviors, etc.
In this paper, we focus on data-driven power grid steady-state modeling that learns mappings and rules among the power grid from operational data such as voltage and power injection, e.g., data-driven power flow linearization~\cite{liu2018data}, external power network equivalence~\cite{yu2017optimal}, optimal reactive power injections learning~\cite{jalali2019designing}, operational security region learning~\cite{cremer2019optimization}, etc. 
Data-driven models are preferred when physical knowledge are too complex~\cite{cremer2019optimization,thams2019efficient} or unavailable~\cite{yu2017optimal,jalali2019designing,zheng2018svm}, or when the power grid analysis involves non-parametric factors such as renewable uncertainties~\cite{babaei2019data,geng2019data} or human behaviors~\cite{bui2019double,carriere2019integrated}.
Nonetheless, the accuracy of data-driven models is usually evaluated empirically on certain testing data and is often unreliable over unfamiliar input data. 
Some recent works address this problem by integrating partial physical knowledge into data-driven models~\cite{cremer2019optimization, yu2017optimal, karagiannopoulos2019data} that intuitively improve the model accuracy.
Some works exploit the vulnerability of data-driven models~\cite{chen2019exploiting} and enhance the robustness of the models by learning under some adversarial cases~\cite{bor2019adversarial}.
Other works predict the range of model inputs that incurs a certain type of classification output, to increase the interpretability of machine learning models~\cite{venzke2019verification}.
However, these approaches are unable to provide any theoretical conclusions about the generalization error of the data-driven model, i.e., the model error over all possible model outcomes.

The inability to bound the generalization error is a significant obstacle that prevents widespread adoption of data-driven models in the power industry. Power system operators must estimate the worst-case model outcome (i.e., the upper bound of errors) for robust operation, and must be able to explain and reinforce the data-driven model with their physical knowledge of the system.  
Besides, how much data is required to bound a data-driven model to the desired accuracy is yet an open question~\cite{ashtiani2018nearly}, as it closely relates to the data quality and model complexity. Simply using more training data may not be ideal due to the trade-off between stable performance and real-time adaptivity~\cite{bhattarai2019big}, i.e., do we want to find a model that best fits all known data, or do we only use recent datasets so the model better adapts to changes in the system states or ambient conditions.


We address the challenges as mentioned above by quantifying a theoretical bound on generalization errors based on the Rademacher complexity theory~\cite{koltchinskii2001rademacher}. 
Rademacher complexity is a notation of complexity that measures the richness of a class of real-valued functions.
Additionally, the bound of Rademacher complexity has been further tightened based on the fact that the worst-case performance of a data-driven model is within a more favorable sub-function class~\cite{cortes2013learning,oneto2019local}.
The formulation is interpretable and has been applied to derive generalization bounds in classification problems~\cite{maximov2018rademacher}.
However, existing Rademacher complexity bounds for regression problems are theoretically loose~\cite{mohri2018foundations}, and cannot consider the integration of any physical knowledge. We propose a modified Rademacher complexity (MRC) generalization bound that is tighter for regression problems and incorporates physical knowledge.
The framework proposed in this work only focuses on bounding the error of regression problems. 
We will also briefly discuss how to extend the framework to classification problems in Section~\ref{section_other}.
We derive the bound using the modified logarithmic Sobolev inequality~\cite{boucheron2013concentration} that satisfies the conditions of regression problems (where the output of the regression is continuous without explicit bound).
The contribution of our work is summarized as follows:
\begin{enumerate}
    \item We provide a tighter bound than conventional Rademacher complexity bound for the generalization error in regression problems, which are typical for conducting data-driven power grid analysis.
    \item We show how much data is required to theoretically bound a data-driven model under all possible training and testing datasets drawn from an underlying distribution, which is critical to guarantee the reliability of data-driven models in power grid applications.
    \item We quantify how physical knowledge can reduce generalization error by incorporating physical knowledge as problem constraints in the MRC model.
    \item Our method applies in power flow analysis and in external network equivalence, which involves the use of linear regression and support vector regression (SVR).
\end{enumerate}

The remainder of this paper is organized as follows. 
Section \ref{section_problem_definition} introduces the preliminaries and the problem statements.
In section \ref{section_MRC}, we propose the MRC generalization error bound, which includes the theoretical results and the computation methods.
Section \ref{section_branch_flow} and section \ref{section_external_network} provide the case studies.
Finally, section \ref{section_conclusion} draws the conclusions.


\section{Preliminaries and Problem Statement}\label{section_problem_definition}
\subsection{Data-driven Power Grid Modeling}
\begin{figure}[htb!]
	\centering
		\includegraphics[width=0.9\linewidth]{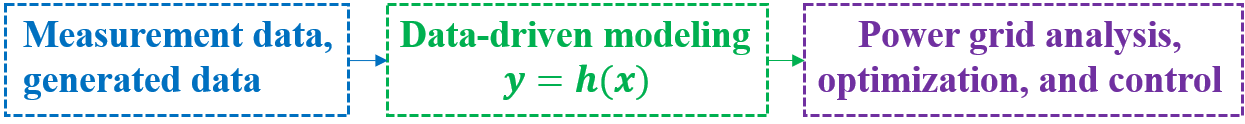}
	\caption{
	Data-driven approach in power grid applications.
	}
	\label{fig_data-driven_problem}
\end{figure}
In this work, we focus on the problem of data-driven power grid steady-state modeling.
This type of problem starts from acquiring data either generated or from measurement devices.
Then, a mapping function is learned in the formulation of $y=h(\bm{x})$.
The input $\bm{x}$ can be power grid operational data such as voltage, power injection, branch flow, etc.
The output $y$ can be some power grid operational data~\cite{liu2018data,yu2017optimal}, some optimal control rules~\cite{jalali2019designing,karagiannopoulos2019data}, or the current power grid status~\cite{cremer2019optimization}.
At last, the data-driven models are used in power system analysis, optimization, and control problems.
The above process is shown in \figurename~\ref{fig_data-driven_problem}.
In this work, we mainly address the data-driven modeling part to evaluate the generalization error of function $y=h(\bm{x})$.

\subsection{Generalization Error}
We start by introducing the formal definition of generalization error and the probably approximately correct (PAC) learning framework, firstly proposed by Valiant~\cite{valiant1984theory}. 
Subsequently, further work on generalization error theory includes Vapnik Chervonenkis (VC) dimension~\cite{vapnik1971uniform}, PAC-Bayes~\cite{seeger2002pac}, and Rademacher complexity~\cite{koltchinskii2001rademacher} have been established over the PAC learning framework.
\begin{definition}[Generalization error]
\label{def_ge}
  Given a sample $x\in \mathcal{X}$ from an underlying distribution $\mathcal{D}$, a hidden ground truth mapping $f(\cdot)$, and a hypothesis class (a.k.a. a data-driven model class such as linear regression) $h \in \mathcal{H}$, the generalization error, $L(h)$, is:
  \begin{equation}\label{eq_GenError}
    L\left(h\right)=\underset{\bm{x} \sim \mathcal{D}}{\mathrm{E}}\big[l(h,x)\big],
  \end{equation}
  where $\mathrm{E}$ denotes the expected value of random variables, and $l(h,x)$ is the loss function mapping from $\mathcal{H}\times\mathcal{X}\rightarrow\mathbb{R}$.
\end{definition}

The generalization error measures the performance over an unseen data distribution $\mathcal{D}$ and cannot be precisely computed.
The training datasets are drawn fromt the unseen data distribution $\mathcal{D}$.
In this paper, we choose to bound the generalization error to evaluate the model performance.
The loss function calculates the loss from an input hypothesis $h$ and the sample point $x$.
In this paper, we use the absolute error as the loss function:
\begin{equation}
  l(h,x)=\left|f(x)-h(x)\right|.
\end{equation}

\begin{definition}[Empirical error]
\label{def_empirical}
The empirical error $\hat{L}_{\bm{x}}(h)$ over the $m$ training samples $\{x_i\}_{i=1}^m$ given a hypothesis $h$ is:
\begin{equation}
    \hat{L}_{\bm{x}}(h)=\frac{1}{m}\sum_{i=1}^{m}l(h,x_i)
\end{equation}
\end{definition}
\noindent
Note that although $f(\cdot)$ is unknown, $f(x_i)$ is known for training sample $x_i$.
Our MRC theory is based on the probably approximately correct (PAC) learning framework, firstly proposed by Valiant~\cite{valiant1984theory}.
The PAC learning framework defines the learnability of a hypothesis class:
\begin{definition}[PAC learning]
  \label{def_PAC}
  A hypothesis class $\mathcal{H}$ is said to have generalization error $\epsilon$ with probability $1-\delta$, if under the unknown distributions $\mathcal{D}$, the generalization error $L\left(h\right)$ of any hypothesis $h\in \mathcal{H}$ satisfies the following inequality:
  \begin{equation}\label{eq_PAC}
    \underset{\bm{x}\sim{\mathcal{D}}}{\mathop{\Pr }}\,\left[ L\left( {h} \right)\le \epsilon  \right]\ge 1-\delta.
  \end{equation}
\end{definition}
\noindent
The PAC learning framework utilizes the probably (at least $1-\delta$ probability) approximately correct (at most $\epsilon$ error under unknown distributions) concept to describe the generalization performance of a hypothesis.


The VC dimension is a pioneering theory to measure the complexity of a hypothesis class~\cite{vapnik1971uniform} but does not consider sample distributions, which makes VC bounds loose with a conservative estimation of the generalization error~\cite{bartlett2005local}.
PAC-Bayes bound is a generic theory to evaluate generalization errors in a Bayesian learning framework, which has been extensively explored recently~\cite{germain2016pac,reeb2018learning,holland2019pac}.
Different from other theories that bound the generalization error $L\left( {h} \right)$,  PAC-Bayes bounds the error according to a posterior distribution $\mathcal{\hat{\pi }}$ over $\mathcal{H}$: $\mathbf{E}_{h\sim\mathcal{\hat{\pi }}}L\left(h\right)$,
where parameters in $h$ are according to the posterior distribution $\mathcal{\hat{\pi }}$.
This approach for measuring the bound, though can help to develop better learning algorithms, may lose their physical meaning in practical engineering applications.

The Rademacher complexity bound uses the information of the sample distribution and derives a tighter bound compared to VC dimension~\cite{mohri2018foundations}. It also directly bounds the $L\left(h\right)$ compared with the PAC-Bayes bound and shows great potential in analyzing data-driven models of engineering problems. The definition of the Rademacher complexity is as follows:
\begin{definition}[Rademacher complexity~\cite{koltchinskii2001rademacher}]
  \label{def_empirical_rademacher}
  Given samples $\bm{x}=\left\{x_{i}\right\}_{i=1}^{m} \sim \mathcal{D}$, the Rademacher complexity of a hypothesis class $\mathcal{H}$ is
    \begin{equation}\label{eq_rademacher}
    \mathfrak{R}(\mathcal{H})=\underset{\bm{x} \sim D}{\mathrm{E}}\left[\widehat{\mathfrak{R}}(\mathcal{H})\right].
  \end{equation}
  where $\widehat{\mathfrak{R}}(\mathcal{H})$ is the empirical Rademacher complexity based on the loss function of the samples, i.e.:
    \begin{equation}\label{eq_empirical_rademacher}
    \widehat{\mathfrak{R}}(\mathcal{H})=\underset{\bm{\sigma}}{\mathrm{E}}\left[\sup _{h \in \mathcal{H}} \frac{1}{m} \sum_{i=1}^{m} \sigma_{i} l\left(h,x_{i}\right)\right].
  \end{equation}
  where $\boldsymbol{\sigma}=\left(\sigma_{1}, \ldots, \sigma_{m}\right)^{T}$ are i.i.d. random variables with $\operatorname{Pr}\left[\sigma_{i}=1\right]=\operatorname{Pr}\left[\sigma_{i}=-1\right]=0.5$, and $\sup$ denotes the supremum of the formulation.
\end{definition}
The Rademacher complexity measures the expected ($\mathrm{E}_{\mathbf{\sigma}}$) richness of a function family $\mathcal{H}$: how well functions in $\mathcal{H}$ can best correlate ($\sup _{h \in \mathcal{H}}$) with random noise ($\sigma_{i}$).
However, the Rademacher complexity bound theory primarily handles classification problems and is less often applied to regression problems~\cite{kuck2018approximate,maximov2018rademacher,oneto2019local}, while numerical calculation strategies of Rademacher complexity is seldom developed.
We therefore propose an MRC generalization bounds based on the modified logarithmic Sobolev inequality~\cite{boucheron2013concentration}, to provide a tighter bound for regression problems. We show that although Rademacher complexity is mostly developed for classification problems, the MRC is also suitable for the regression problem. 

\subsection{Physical Knowledge in Data-driven Models}


The physical knowledge space $\mathcal{P}$ describes what we know about the target system in addition to measurements.
For instance, in power grid analysis, this includes the topology, line impedance, or physical principles such as power flow models. 
Incorporation of physical knowledge stabilizes data-driven model performance as it reduces model complexity by intersecting the hypothesis class space with the physical knowledge space $\mathcal{H}\cap \mathcal{P}$ as shown in \textbf{Definition~\ref{def_physical}}.

\begin{definition}[Data-driven models with physical knowledge]
  \label{def_physical}
  Given samples $\bm{x}=\left\{x_{i}\right\}_{i=1}^{m} \sim \mathcal{D}$ and the physical knowledge based space $\mathcal{P}$ derived from some physical rules, the model aims at finding the hypothesis $h\in \mathcal{H}\cap\mathcal{P}$ while minimizing the empirical error $\hat{L}_{\bm{x}}(h)$.
\end{definition}

In practice, physical knowledge $\mathcal{P}$ can be described as data-driven model parameter constraints (i.e. $p(\bm{x},h)\leq0$) during the training stage.
For example, Yu~\textsl{et al.}~\cite{yu2017optimal} added external network parameter constraints (derived from system maximum and minimum operating modes) to least squares regression models; Karagiannopoulos \textsl{et al.}~\cite{karagiannopoulos2019data} added box constraints (derived from imbalanced control penalties) to SVR models.

\subsection{Problem Statement}
We seek a generalization error bound $\epsilon$ that reflects the underlying hypothesis space $\mathcal{H}$, the physical knowledge space $\mathcal{P}$, and the training samples $\bm{x}$, formally stated as follows:
\begin{equation}\label{eq_bound}
    \epsilon=
    \mathbf{MRC}\left(\mathcal{H}\cap\mathcal{P},\bm{x}\right)
\end{equation}

\begin{figure}[ht!]
	\centering
	\begin{minipage}[ht]{4.2cm}
		\centering
		\includegraphics[width=.85\linewidth]{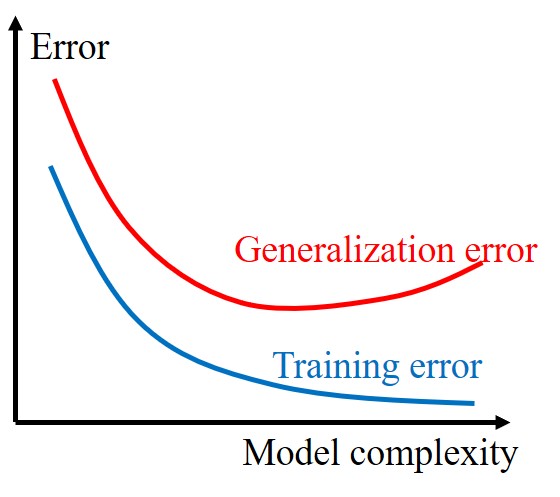}
		\centerline{\footnotesize{(a)}}
		\label{fig_illustrative_a}
	\end{minipage}
	\begin{minipage}[ht]{4.2cm}
		\centering
		\includegraphics[width=.8\linewidth]{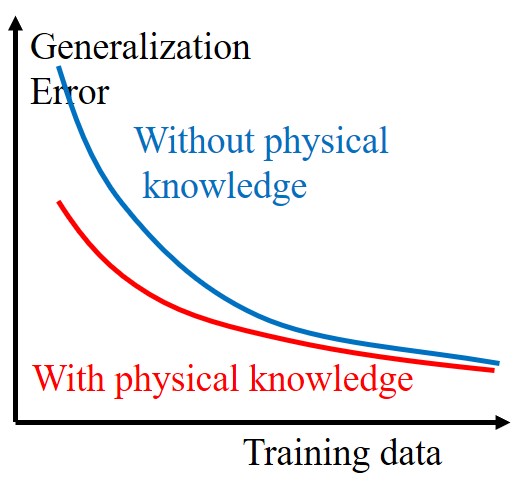}
		\centerline{\footnotesize{(b)}}
		\label{fig_illustrative_b}
	\end{minipage}
	\vspace{-.3cm}
	\caption{Illustration of how data-driven model complexity, physical knowledge, and training data affect the generalization error. (a) The effects of data-driven model complexity. (b) The effects of physical knowledge and training data.}
	\label{fig_illustrative}
\end{figure}

Our proposed approach provides a way to benchmark the choice of model, training data size, and physical knowledge in data-driven grid modeling. For example, a linear regression model with more regressors achieves better training results, but its generalization error may increase due to over-fitting, as illustrated in~\figurename\ref{fig_illustrative}(a); while some well established physical knowledge improves the generalization error, this improvement diminishes with larger training datasets, as illustrated in~\figurename\ref{fig_illustrative}(b).
The above statements will also be validated later in the case study.


\section{The Modified Rademacher Complexity Bound}
In this section, we introduce the proposed MRC bound.
We first present the theoretical derivation of the bound.
Then, we adopt an iteration strategy to tighten the bound further.
At last, we provide the approach to evaluate the bound numerically.
\label{section_MRC}
\subsection{Theoretical Result}
\label{section_theo_deri}
We start by introducing the main Theorem, which states that the empirical error of any hypothesis $h$ in a class $\mathcal{H}(e)$ (which will be defined later) can be bounded based on the sample-based empirical Rademacher complexity: 
\begin{theorem}[Empirical modified Rademacher complexity bound]
  \label{theo_empirical_MRC}
  We can bound the generalization error, for $\forall \delta \in (0,1)$ and $\forall h \in \mathcal{H}(e)$, with at least $1-\delta$ probability:
  \begin{equation}\label{eq_empir_MRC}
    L\left(h\right) \leq \hat{L}_{\bm{x}}(h)+2\widehat{\mathfrak{R}}(\mathcal{H}(e))+3e\sqrt{(2\log\frac{2}{\delta })/m},
  \end{equation}
\end{theorem}
\noindent
where $\hat{L}_{\bm{x}}(h)$ is the empirical error as in \textbf{Definition~\ref{def_empirical}}, $\delta$ is the PAC learning probability as in \textbf{Definition~\ref{def_PAC}}, $m$ is the sample size, 
and $e$ is an upper bound of the mean square error (MSE) over arbitrary $m$ testing samples $\bm{x}=\{x_i\}^m\sim\mathcal{D}$ after arbitrary $m$ training samples $\forall\bm{x'}=\{x'_i\}^m\sim\mathcal{D}$, defined below as:
\begin{equation}
\begin{aligned}\label{eq_condition_MRC}
  &\frac{1}{m}\sum_{i=1}^{m}l(h,x_i)^2 \leq e^2, h\in \mathcal{H}(e)\\
  &\mathcal{H}(e):=\left\{h = \mathrm{Tr}(h',\bm{x'}),\forall h'\in\mathcal{H}\right\},
\end{aligned}
\end{equation}
where $\mathrm{Tr}(\cdot)$ is the training of the data-driven model, and $\mathcal{H}(e)$ denotes the narrowed hypothesis class bounded by the MSE of $e^2$ after training any $m$ samples. The Rademacher complexity generalization can be intuitively interpreted as follows:
The generalization error of a trained hypothesis class $\mathcal{H}(e)$ has a high probability of being large when:
1) the empirical error (training error) is large (the first part of~\eqref{eq_empir_MRC});
2) the hypothesis class $\mathcal{H}(e)$ is ``complex'' (the second part of~\eqref{eq_empir_MRC});
3) the training sample size is small (the third part of~\eqref{eq_empir_MRC}).
The first part of~\eqref{eq_empir_MRC} can be obtained after training.
The third part of~\eqref{eq_empir_MRC} can be directly calculated by the value of $e$, $\delta$, and $m$.
We will introduce how to evaluate the second part in Section~\ref{section_numerical_solution}.

Assumption \eqref{eq_condition_MRC} denotes that the MSE is bounded after training arbitrary $m$ samples.
It is more flexible than assuming the maximum value of the loss function $l(h,x_i)$ is bounded in the traditional Rademacher complexity bound~\cite{koltchinskii2001rademacher}.
This is reasonable in classification problems yet difficult to be applied in regression problems, because $l(h,x_i)$ may be extremely large.
In our MRC bound, however, one only needs to set an upper bound of MSE, which is more applicable in regression problems.
We use an iterative strategy to configure the upper bound for MSE, as shown in Section \ref{section_iteration}.

The proof of \textbf{Theorem~\ref{theo_empirical_MRC}} is based on two key results. First, we show that the empirical error can be bounded by the Rademacher complexity (which is the expectation of the empirical Rademacher complexity) under $1-\delta$ probability as
\begin{equation}\label{eq_MRC}
  L\left(h\right) \leq \hat{L}_{\bm{x}}\left(h\right)+2\mathfrak{R}(\mathcal{H}(e))+e\sqrt{(2\log\frac{2}{\delta })/m}.
\end{equation}
whose proof is primarily based on the use of the modified logarithmic Sobolev inequality from Section 6.4 of~\cite{boucheron2013concentration}.
Starting with this inequality, we can derive the bound of the generalization error $L\left(h\right)$ from assumption \eqref{eq_condition_MRC}.
Still, the result in~\eqref{eq_MRC} cannot be used directly to quantify the generalization error, because we cannot obtain the value of Rademacher complexity from the training samples $\bm{x}$.
Instead, we show that the Rademacher complexity can be bounded by the empirical Rademacher complexity, under $1-\delta/2$ probability:
\begin{equation}\label{eq_empir_MRC1}
    \mathfrak{R}(\mathcal{H})
    \leq
    \widehat{\mathfrak{R}}(\mathcal{H})
    +e\sqrt{(2\log\frac{2}{\delta })/m},
  \end{equation}
which is derived by combining \eqref{eq_empir_MRC1} and \eqref{eq_MRC}. 
The full proof is listed in the Appendix.

\begin{remark}[Physical knowledge]
The physical knowledge $\mathcal{P}$ narrows the hypothesis space and reduces the Rademacher complexity: $\widehat{\mathfrak{R}}(\mathcal{H\cap\mathcal{P}}) \leq \widehat{\mathfrak{R}}(\mathcal{H})$.
The bound in \eqref{eq_empir_MRC} is also valid for $\widehat{\mathfrak{R}}(\mathcal{H\cap\mathcal{P}})$.

\end{remark}


\subsection{Iteration Strategy}
\label{section_iteration}

In this section, we implement an iteration strategy for the configuration of the upper bound of the MSE $e^2$.
The basic idea is to assume an initial $e$ and narrow the $e$ and $L(h)$ through iterations.
\begin{enumerate}
    \item \textbf{Step 1}: Assume a large enough initial $e$. Set the maximum iteration steps $I$ and the current iteration step $i=1$. 
    \item \textbf{Step 2}: Evaluate the generalization bound $L\left(h\right)$ using \textbf{Theorem~\ref{theo_empirical_MRC}} (equation~\eqref{eq_empir_MRC}).
    \item \textbf{Step 3}: Update $e$ by $e^{new}\leftarrow kL\left(h\right)$, where $k$ is a constant factor that represents how tight we can bound the MSE ($e^{new}$) of the next iteration from the mean absolute error (MAE) bound ($L(h)$) of this iteration. See Appendix for the configuration of factor $k$. If $e^{new}\textless e$ and $i\textless I$, set $i\leftarrow i+1$ and do \textbf{Step~2}. Else, do \textbf{Step 4}.
    \item \textbf{Step 4}: Output the final generalization bound as $L\left(h\right)$.
\end{enumerate}

Finally, we divide the evaluation of $L(h)$ in equation~\eqref{eq_empir_MRC} (in \textbf{Step~2}) into three parts: the empirical error, the empirical Rademacher complexity, and the randomness of the sample distribution. The iterative framework is shown in \figurename\ref{fig_framework}.

\begin{figure}[htb!]
	\centering
		\includegraphics[width=.9\linewidth]{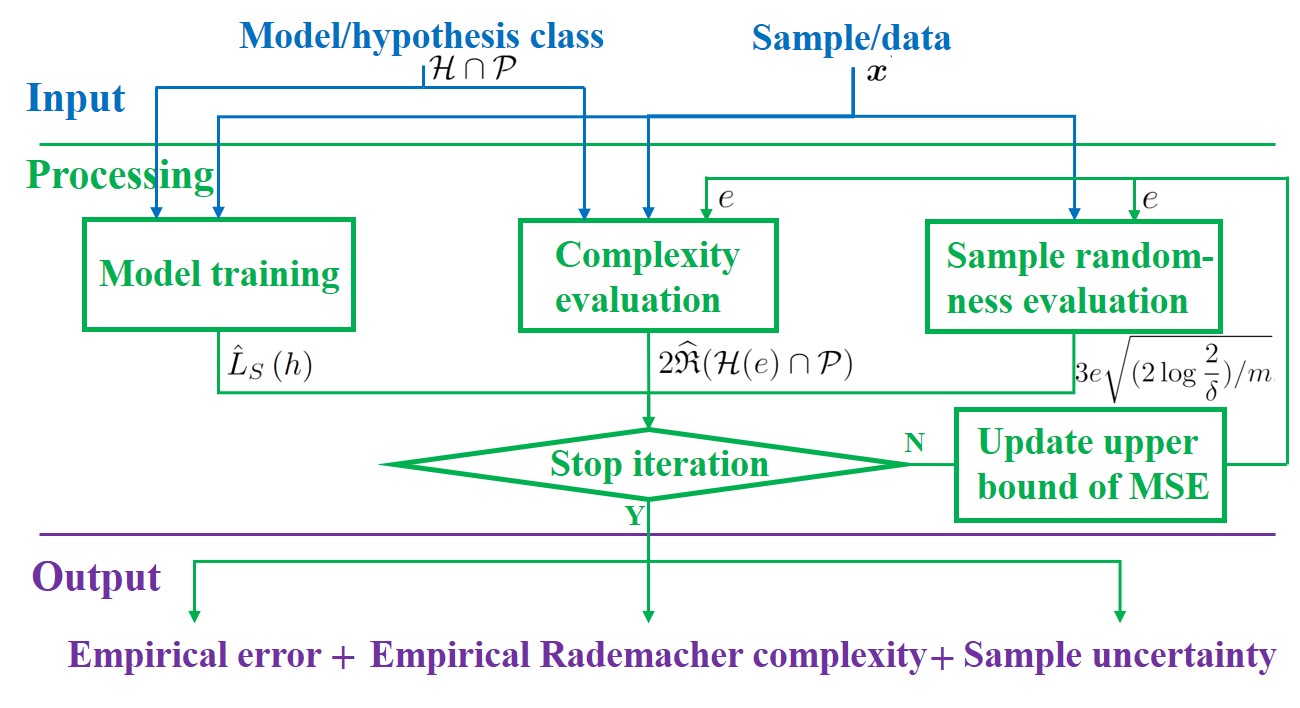}
	\caption{
	Flowchart of modified Rademacher complexity bound evaluation.
	}
	\label{fig_framework}
\end{figure}

\subsection{Computation of Empirical Rademacher Complexity}
\label{section_numerical_solution}
To calculate the empirical Rademacher complexity as in \textbf{Definition \ref{def_empirical_rademacher}}, we generate $n$ samples of the Rademacher variable vector $\bm{\sigma}_j \in \{-1,1\}^m$ each with $m$ elements. The empirical Rademacher complexity can thus be calculated using the generated samples as

\begin{shrinkeq}{-0.4ex}
\begin{subequations}\label{eq_rademacher_average}
  \begin{equation}\label{eq_rademacher_average_a}
    \widehat{\mathfrak{R}}(\mathcal{H}(e)\cap \mathcal{P})
    \approx
    \frac{1}{n}
    \sum_{j=1}^{n}
    {
    Sup\left(\boldsymbol{\sigma_j},\bm{x},l\right)
    }
  \end{equation}
  \begin{equation}\label{eq_rademacher_average_b}
    Sup\left(\boldsymbol{\sigma_j},\bm{x},h\right)
    =
    \sup _{h \in \mathcal{H}(e)\cap \mathcal{P}} \frac{1}{m} \sum_{i=1}^{m} \sigma_{ij} l\left(h,x_{i}\right),
  \end{equation}
\end{subequations}
\end{shrinkeq}
where $\sigma_{ij}$ denotes the $i$th element of the $j$th vector of Rademacher variables $\boldsymbol{\sigma_j}$.
Afterwards, we cast the sup in \eqref{eq_rademacher_average_b} as a maximization problem:
\begin{shrinkeq}{-0.4ex}
\begin{subequations}\label{eq_opt_initial}
  \begin{equation}\label{eq_opt_initiala}
    \underset{h}{\max}
    \sum_{i=1}^{m}
    \sigma_{ij}
    l(h,x_i)
  \end{equation}
  \begin{equation}\label{eq_opt_initialb}
    \sum_{i=1}^{m}
    \left[f(x_i)-h(x_i)\right]^2
    \leq me^2
  \end{equation}
  \begin{equation}\label{eq_opt_initialc}
    p(\bm{x},h)\leq0,
  \end{equation}
\end{subequations}
\end{shrinkeq}
where the maximum value of \eqref{eq_opt_initiala} is $(1/m)Sup\left(\boldsymbol{\sigma_j},\bm{x},h\right)$, \eqref{eq_opt_initialb} denotes the upper bound of MSE as assumed in~\eqref{eq_condition_MRC}, and~\eqref{eq_opt_initialc} represents the physical constraints.
The physical constraints may relate to the samples $\bm{x}$ or the parameters of the hypothesis $h$.
Specified formulations of~\eqref{eq_opt_initialc} will be demonstrated in Section~\ref{section_branch_flow} and Section~\ref{section_external_network}.

Then, we cast the absolute value formulations in \eqref{eq_opt_initial} as a differentiable problem, by introducing complementary auxiliary variables:
\begin{subequations}\label{eq_auxiliary}
  \begin{equation}
  \left|f(x_i)-h(x_i)\right|=d_{i}^{+}+d_{i}^{-},
  \end{equation}
  \begin{equation}
  f(x_i)-h(x_i)=d_{i}^{+}-d_{i}^{-}.
  \end{equation}
\end{subequations}
Applying the big M method, problem \eqref{eq_opt_initial} becomes:
\begin{shrinkeq}{-0.4ex}
\begin{subequations}\label{eq_opt_slow}
  \begin{equation}\label{eq_opt_slowa}
    \underset{h,d_{i}^{+},d_{i}^{-},u_{i}}{\max}
    \sum_{i=1}^{m}
    \sigma_{ij}
    \left(d_{i}^{+}+d_{i}^{-}\right)
  \end{equation}
  \begin{equation}\label{eq_opt_slowa}
    h(x_i)-f(x_i)=d_{i}^{+}-d_{i}^{-}
    \text{, }\forall i
  \end{equation}
  \begin{equation}\label{eq_opt_slowb}
    0\leq d_{i}^{+}\leq Mu_{i}
    \text{, }\forall i\text{ s.t. }\sigma_{ij}=+1
  \end{equation}
  \begin{equation}\label{eq_opt_slowc}
    0\leq d_{i}^{-}\leq M(1-u_{i})
    \text{, }\forall i\text{ s.t. }\sigma_{ij}=+1
  \end{equation}
  \begin{equation}\label{eq_opt_slowd}
    \sum_{i=1}^{m}
    \left(d_{i}^{+}+d_{i}^{-}\right)^2
    \leq me^2
  \end{equation}
  \begin{equation}\label{eq_opt_slowe}
    p(\bm{x},h)\leq0,
  \end{equation}
\end{subequations}
\end{shrinkeq}
where $M$ denote the big M value, $u_i$ denotes the $i$th 0-1 decision variable.
The big M related constraints \eqref{eq_opt_slowb} and \eqref{eq_opt_initialc} ensure that at least one of $d_i^+$ and $d_i^-$ should be zero.
This is a necessary condition of the transformation in \eqref{eq_auxiliary}.
We only need to add the above constraints when $\sigma_{ij}=+1$.
Because when $\sigma_{ij}=-1$, minimizing $d_{i}^{+}+d_{i}^{-}$ will guarantee no violation of \eqref{eq_opt_slowb} and \eqref{eq_opt_initialc}~\cite{boyd2004convex}.
Note that constraint \eqref{eq_opt_slowd} is non-convex.
We adopt the following theorem to convexify the problem.
\begin{proposition}\label{theo_MAE}
  Assume after training of $m$ samples, the MAE has an upper bound for $\forall S\sim D$:
  \begin{equation}\label{eq_condition_MAE}
    (1/m)\sum_{i=1}^{m}{\left| f\left( {{x}_{i}} \right)-h\left( {{x}_{i}} \right) \right|} \leq e.
  \end{equation}
Then, the empirical Rademacher complexity under assumption \eqref{eq_condition_MAE} ($\widehat{\mathfrak{R}}'(\mathcal{H}(e)\cap \mathcal{P})$) is the upper bound of the empirical Rademacher complexity under assumption \eqref{eq_condition_MRC} ($\widehat{\mathfrak{R}}(\mathcal{H}(e)\cap \mathcal{P})$):
  \begin{equation}\label{eq_theo_MAE}
    \widehat{\mathfrak{R}}(\mathcal{H}(e)\cap \mathcal{P})
    \leq
    \widehat{\mathfrak{R}}'(\mathcal{H}(e)\cap \mathcal{P}).
  \end{equation}
\end{proposition}
\begin{proof}
  Denote $\left| f\left( {{x}_{i}} \right)-h\left( {{x}_{i}} \right) \right|=d_i$.
  $\forall h$ satisfies \eqref{eq_condition_MRC}, we have:
  \begin{equation}
    \frac{1}{m}
    \sum_{i=1}^{m}d_i
    \leq
    \left(
      \frac{1}{m}
      \sum_{i=1}^{m}{d_i}^2
    \right)
    ^
    {1/2}
    \leq
    e,
  \end{equation}
  where the first inequality holds by the general means inequality~\cite{bullen2013means}, the second inequality is from \eqref{eq_condition_MRC}.
  Thus, any hypothesis $h$ that satisfies \eqref{eq_condition_MRC} also satisfies \eqref{eq_condition_MAE}, which proves \eqref{eq_theo_MAE}.
\end{proof}
From \textbf{Proposition \ref{theo_MAE}}, we evaluate $\widehat{\mathfrak{R}}'(\mathcal{H}(e)\cap \mathcal{P})$ by substituting constraint \eqref{eq_opt_slowd} to:
\begin{equation}\label{eq_subs_MAE}
  \sum_{i=1}^{m}
    \left(d_{i}^{+}+d_{i}^{-}\right)
    \leq me.
\end{equation}
In many applications, $h\left(x_i\right)$ and $p(\bm{x},h)$ in \eqref{eq_opt_slow} are linear or quadratic corresponding to the parameters of $h$, thus the numerical evaluation of $\widehat{\mathfrak{R}}'(\mathcal{H}(e)\cap \mathcal{P})$ turns into a mixed-integer linear programming (MILP) problem or a mixed-integer quadratically-constrained programming (MIQCP) problem, which can be solved by many commercial solvers such as Gurobi and Cplex.

\subsection{Other Machine Learning Problems}
\label{section_other}
Our framework is generally applicable to supervised learning problems. 
That is, we can easily extend the MRC framework to classification problems, where all the process is the same with regression problems.
Recall that the MRC is designed to address the unbounded loss value of regression problems.
Thus, we can even more easily handle classification problems because they have bounded outputs.
Consider a binary classification problem with output set $\{0, 1\}$.
Then we can set the initial $e=1$ and start the iteration in Section~\ref{section_iteration}.
Still, our framework is not applicable for unsupervised learning or reinforcement learning problems.
Future research will focus on provable guarantees on those problems, especially with power system knowledge embedding.

\section{Branch Flow Linearization}
\label{section_branch_flow}
\subsection{Problem Formulation}
The branch flow linearization problem searches for an optimal mapping function between the branch flow and the voltage of the connected buses from historical operational data~\cite{liu2018data}.
The linearized branch flow function has better performance in terms of computational speed and convergence guarantee in power flow applications, such as optimal power flow computation~\cite{hu2019ensemble}, low observability state estimation~\cite{donti2019matrix}, and distributed voltage control problems~\cite{magnusson2020distributed}.
For simplicity, we only consider branch flow linearization in this paper, which formulates the basis process of linearizing power flow equations~\cite{li2018data}.
The well-known branch flow equations are as follows:
\begin{shrinkeq}{-0.3ex}
\begin{subequations}\label{eq_branch_flow}
  \begin{equation}
    P_{ij} = 
    g_{i j}\left(v_{i}^{2}-v_{i} v_{j} \cos \theta_{i j}\right)
    -b_{i j} v_{i} v_{j} \sin \theta_{i j}
  \end{equation}
  \begin{equation}
    Q_{ij} = 
    -b_{i j}\left(v_{i}^{2}-v_{i} v_{j} \cos \theta_{i j}\right)
    -g_{i j} v_{i} v_{j} \sin \theta_{i j},
  \end{equation}
\end{subequations}
\end{shrinkeq}
where subscript $i/j$ denote the bus number,
$v_i/v_j$ denote the voltage magnitudes of bus $i/j$,
and $g_{ij}/b_{ij}$, $P_{ij}/Q_{ij}$, and $\theta_{ij}$ denote conductance/susceptance, active/reactive power flow, and voltage angle of branch $(i,j)$, respectively.
In this paper, we use the following linearized branch flow formulations:
\begin{equation}\label{eq_linear_bf}
  h_{ij}^{bf}(v,\theta)
  = \alpha_1^{bf}v_i^2 + \alpha_2^{bf}v_j^2 + \alpha_3^{bf}\theta_i + \alpha_4^{bf}\theta_j + \alpha_5^{bf},
\end{equation}
where $h_{ij}^{bf}(v,\theta)$ denotes the linearized formulation of active or reactive branch flow $P_{ij}$ or $Q_{ij}$, and $\alpha_1^{bf}\sim\alpha_5^{bf}$ denote the parameters of $h_{ij}^{bf}(v,\theta)$.
In \eqref{eq_linear_bf}, branch flow is a linearized form with respect to $v^2$ and $\theta$, which is a recommended representation in ~\cite{yang2018general}.
The parameters of $\alpha_1^{bf}\sim\alpha_5^{bf}$ are computed from historical operational data by the ordinary least squares (OLS) regression algorithm.
Parameters of different systems and different branches are different, to best fit the system and its operational characteristics.

\subsection{Physical Knowledge}

From \eqref{eq_branch_flow} and \eqref{eq_linear_bf}, we can observe that the voltage angle only appears in the form of $\theta_{ij}$.
In other words, the coefficient of $\theta_i$ is the negative of the coefficient of $\theta_j$.
Thus, we adopt the following constraints as the physical constraints of the problem:
\begin{equation}\label{eq_bf_physical1}
  -\Delta^{bf}
  \leq
  \alpha_3^{bf}+\alpha_4^{bf}
  \leq
  \Delta^{bf},
\end{equation}
where $\Delta^{bf}$ is a non-negative small value.
Furthermore, we can obtain the expressions of $\alpha_1^{bf}\sim\alpha_5^{bf}$ by first order expansion, and thus set the maximum and minimum bounds $\overline{B_{i}}^{bf}$ and $\underline{B_{i}}^{bf}$ considering the boundary operational conditions~\cite{yu2017optimal}:
\begin{equation}\label{eq_bf_physical2}
  \underline{B_{i}}^{bf}
  \leq
  \alpha_i^{bf}
  \leq
  \overline{B_{i}}^{bf}
  \text{, }i=1\sim5.
\end{equation}
Finally, set \eqref{eq_linear_bf} as the hypothesis $h$ in \eqref{eq_opt_slow}, and set \eqref{eq_bf_physical1} or \eqref{eq_bf_physical2} as the physical constraints in \eqref{eq_opt_slowe}.
The evaluation of $\widehat{\mathfrak{R}}'(\mathcal{H}(e)\cap \mathcal{P})$ then becomes an MILP problem.

\subsection{Numerical Results}
We generate the data using Monte Carlo simulation, with the aid of MATPOWER 6.0 ~\cite{zimmerman2010matpower}.
Then, all the data are normalized to the interval [0, 1].
The data generation strategy of power system operational data is the same as in~\cite{liu2018data}.
The optimization problem is solved by Gurobi 8.1 with Python interface.
We use the IEEE 118-bus system in this case study.
Set the non-negative small value $\Delta^{bf}=10^{-2}$, the number of Rademacher variables $n=10$, the maximum iteration steps $I=10$, and the PAC learning probability $\delta=0.05$. 




We analyze the generalization error of the reactive flow in Branch \#96, which has the largest training error in all branches.
The generalization error bounds of three cases under different amount of training data are compared: 1) \textbf{NonPhys}: Without any physical knowledge; 2) \textbf{Angle}: Consider the voltage angle constraint \eqref{eq_bf_physical1}; 3) \textbf{Box}: Consider the maximum and minimum parameter bound constraints \eqref{eq_bf_physical2}. \figurename\ref{fig_physical_bf} shows that the generalization error bound of all the three cases decrease as the training data amount increases.

\begin{figure}[hbt!]
	\centering
		\includegraphics[width=.85\linewidth]{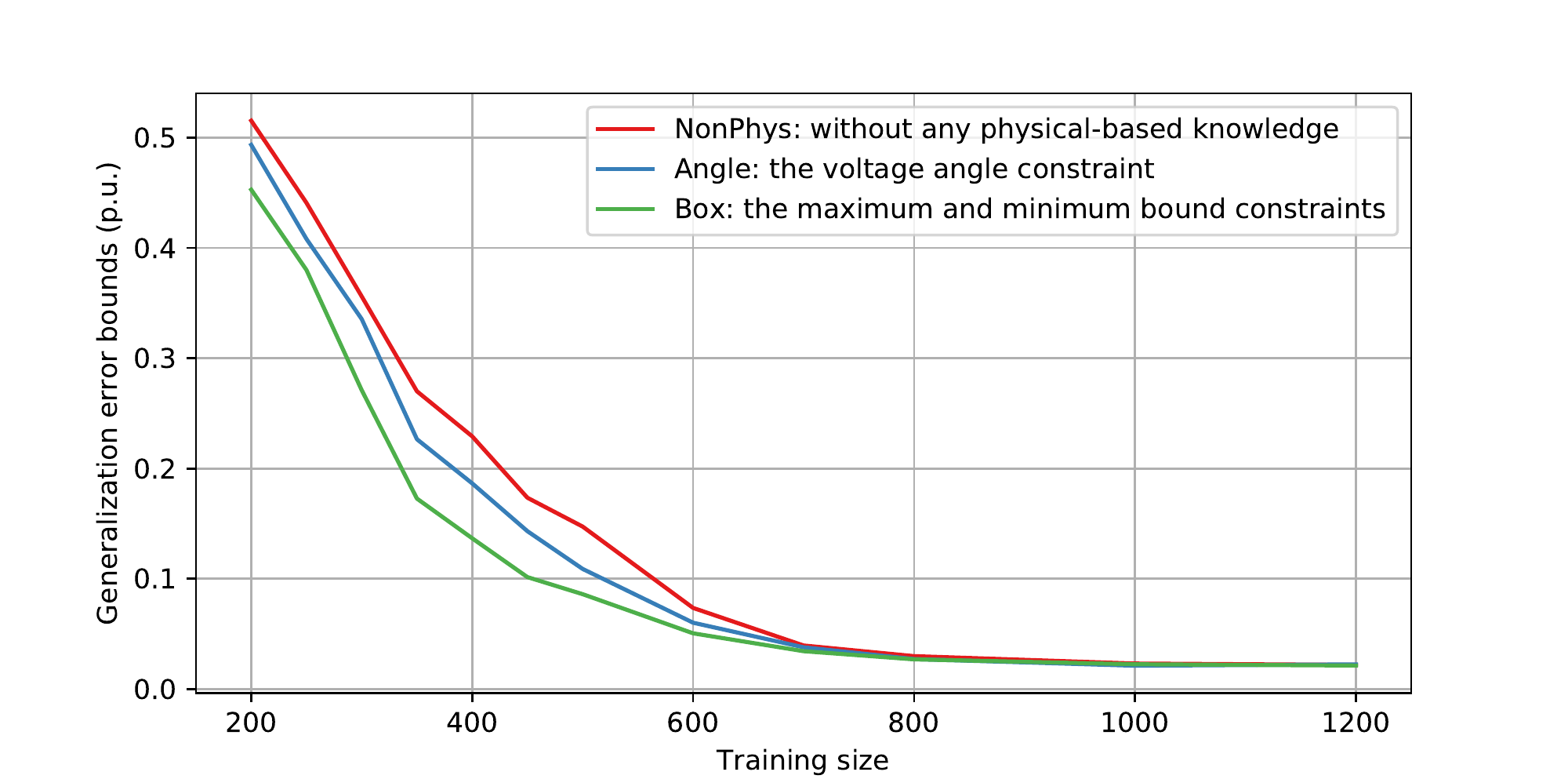}
	\caption{
	Comparisons of different physical knowledge in the case of reactive branch flow of Branch \#96.
	The total generalization bounds of methods \textbf{NonPhys}, \textbf{Angle}, and \textbf{Box} under different amounts of training data are compared. The training size ranges from 200 to 1200.
	}
	\label{fig_physical_bf}
\end{figure}

The addition of physical knowledge significantly decreases the generalization error bounds, especially when the training size is below 700.
Furthermore, \textbf{Box} has a more significant effect of reducing the generalization bounds than \textbf{Angle}.
As shown in \figurename\ref{fig_physical_bf}, to obtain the generalization bound of 0.1 p.u., the effect of physical knowledge \textbf{Angle} is equivalent to nearly 50 snapshots of training data, while the effect of \textbf{Box} is equivalent to nearly 100 snapshots.
In practice, the above results can be used to determine the size of training samples and to design the implementation of physical knowledge.
\figurename\ref{fig_physical_bf} suggests that the training size over 800 has less effect to enhance the generalization behavior.
From the proposed theory, we suggest 800 training sample size is enough for this problem.
In addition, we can also conclude that the physical knowledge \textbf{Box} is far more effective than \textbf{Angle}.

\begin{figure}[hbt!]
	\centering
		\includegraphics[width=.8\linewidth]{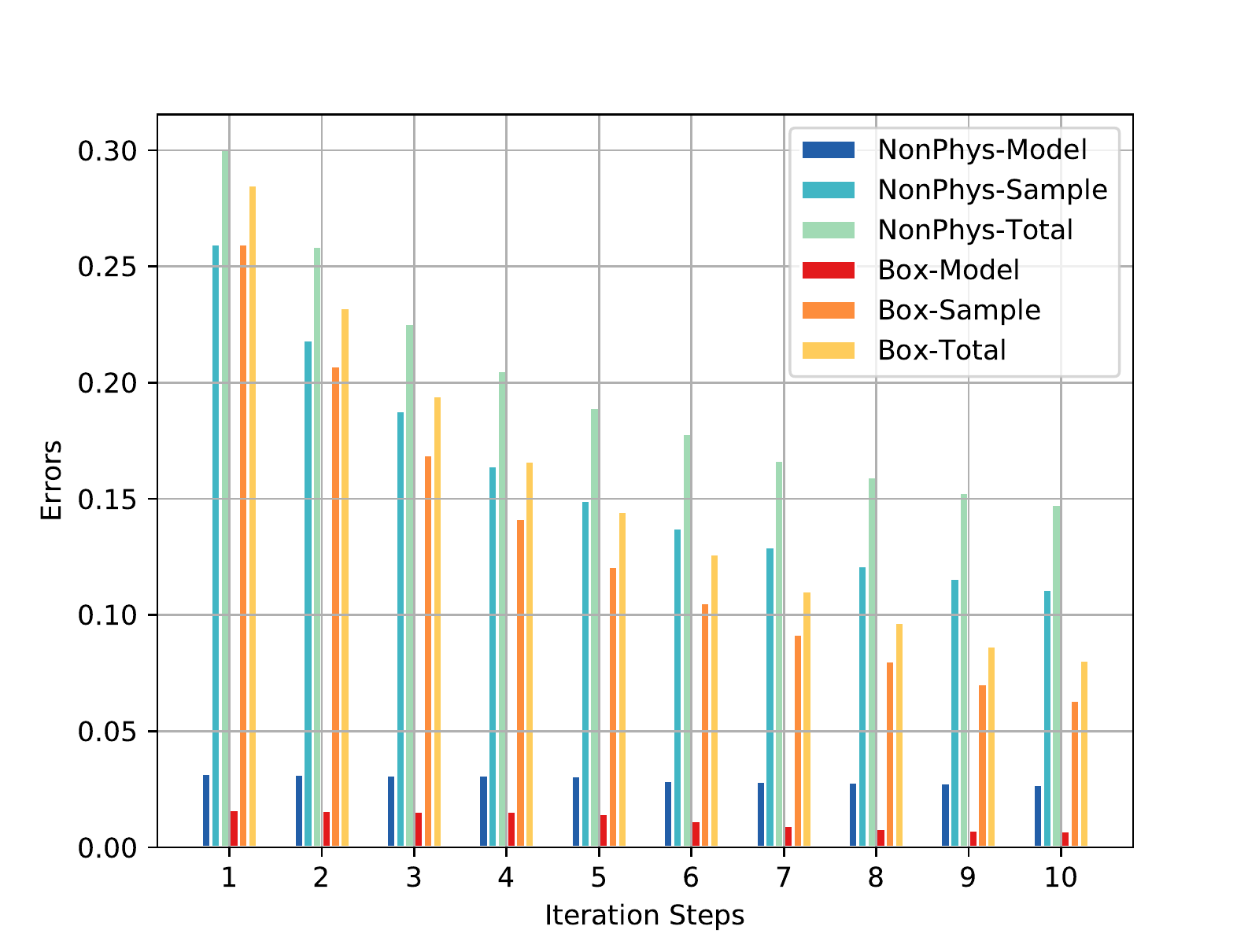}
	\caption{
	Illustration of different parts of the generalization error bounds in each iteration step.
	The bounds include the complexity of the data-driven model (Model), the randomness of sample distribution (Sample), and the total generalization error bound (Total).
	The training size is 500.
	Two cases are compared: \textbf{NonPhys} and \textbf{Box}.
	}
	\label{fig_iteration_bf}
\end{figure}

We then demonstrate different parts of the generalization error bounds in each iteration step in \figurename~\ref{fig_iteration_bf}.
The complexity of the data-driven model, the randomness of the sample distribution, and the total generalization error bound decrease through iterations.
The total generalization bound difference of \textbf{NonPhys} and \textbf{Box} increases through iterations.
At iteration step 1, the difference of the total generalization bounds is small because the initial $e$ is the same so that the randomness of the sample distribution is the same.
Then, the difference of the total generalization bounds results in the difference of $e$ in the next iteration step.
Therefore, the small difference accumulates through iterations and results in a significant difference in~\figurename~\ref{fig_iteration_bf}.

\begin{figure}[hbt!]
	\centering
		\includegraphics[width=1\linewidth]{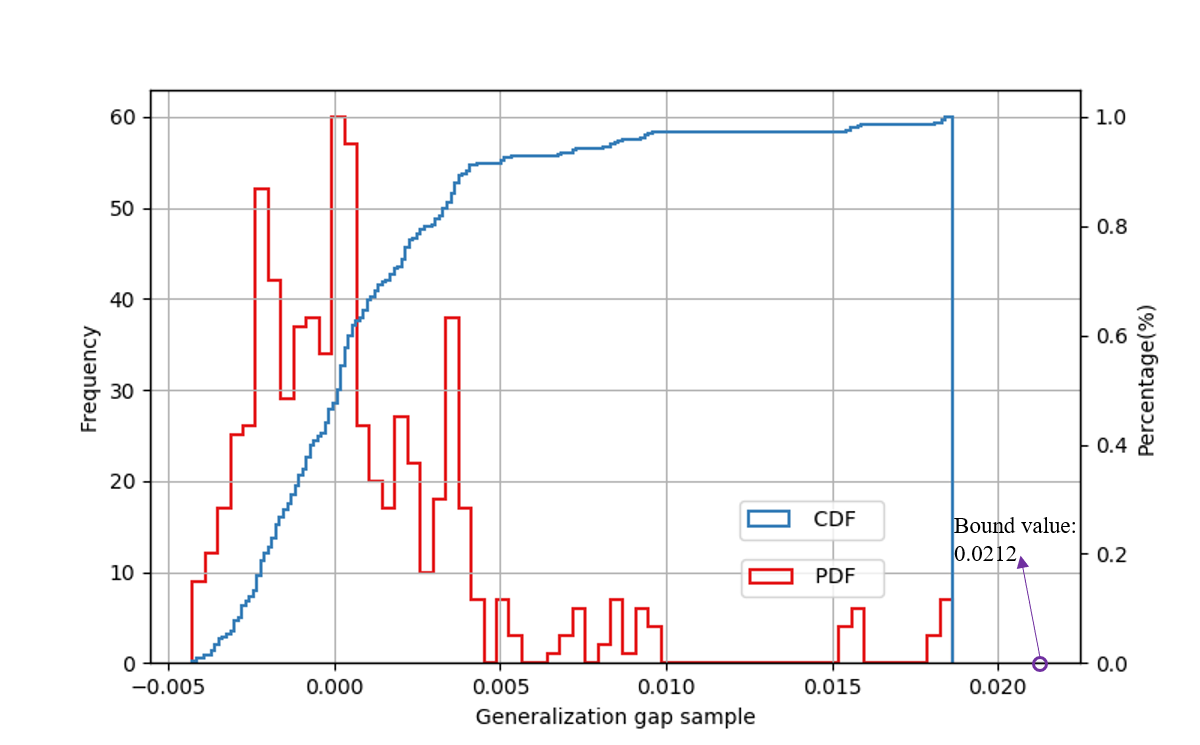}
	\caption{
	The probability density function (PDF) and the cumulative distribution function (CDF) of the generalization gap $\hat{L}_{\bm{x}_{test}}(h)-\hat{L}_{\bm{x}_{train}}(h)$.
	The PDF is shown by frequency (left vertical axis) while the CDF is shown by percentage (right vertical axis).
	20 training datasets are generated with each dataset contains 800 samples. 
	40 testing datasets are generated to evaluate the generalization performance. 
	The bound value computed by~\eqref{eq_empir_MRC_re} is also marked in the figure.
	}
	\label{fig_cdf_pdf_bf}
\end{figure}
Then, we demonstrate the tightness of the bound using some testing scenarios.
Recall that our bound in \textbf{Theorem~\ref{theo_empirical_MRC}} is valid for all possible training and testing data drawn from an underlying distribution $\mathcal{D}$.
To test our bound under different training and testing scenarios, we reformulate our bound as:
\begin{equation}\label{eq_empir_MRC_re}
L\left(h\right)-\hat{L}_{\bm{x}}(h) \leq 2\widehat{\mathfrak{R}}(\mathcal{H}(e))+3e\sqrt{(2\log\frac{2}{\delta })/m},
\end{equation}
where we bound the generalization gap $L\left(h\right)-\hat{L}_{\bm{x}}(h)$ instead.
We first use the original sample $\bm{x}$ to evaluate the bound of the generalization gap.
Then we generate multiple training and testing datasets $\{\bm{x}_{train}^{1}, \bm{x}_{train}^{2}, ...\}$ and $\{\bm{x}_{test}^{1}, \bm{x}_{test}^{2}, ...\}$.
The difference of testing error and training error is calculated by 
$\hat{L}_{\bm{x}_{test}}(h)-\hat{L}_{\bm{x}_{train}}(h)$, where $\bm{x}_{test}$ is one of the testing datasets, and $\bm{x}_{train}$ is one of the training datasets.
To test the performance of bounding large possible generalization gaps, we show the distribution of the generalization gap samples.
The probability density function (PDF), the cumulative distribution function (CDF), and the bound value are shown in \figurename~\ref{fig_cdf_pdf_bf}.
The generalization gap follows a fat-tailed-distribution.
Large errors may occur with little possibility, which is the reason to evaluate the bound rather than the mean value of the generalization gap.
From \figurename~\ref{fig_cdf_pdf_bf}, our result can well bound the generalization gap under the uncertainty of training and testing scenarios.


\section{External Network Equivalent}
\label{section_external_network}
\subsection{Problem Formulation}
In interconnected power systems, the power flow model of external networks can be simplified through an equivalent network representation.
Such equivalence is essential for interconnected systems where the information is not shared, and can also reduce the computational complexity of system optimization problems.
The equivalence model can be used in a wide range of applications such as contingency analysis, optimal power flow dispatch, and static voltage stability analysis, etc.
In this section, we first use the model in~\cite{yu2017optimal} as shown in \figurename\ref{fig_external_model}, where the border PMU data is used to estimate the parameters of the model.
\begin{figure}[hbt!]
	\centering
		\includegraphics[width=.65\linewidth]{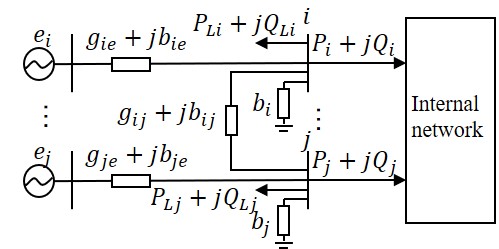}
	\caption{The external network equivalence model in ~\cite{yu2017optimal}.}
	\label{fig_external_model}
\end{figure}

The model can be formulated as follows:
\begin{shrinkeq}{-0.5ex}
\begin{subequations}\label{eq_EX_model}
  \begin{equation}\label{eq_EX_model_a}
    P_{i}=-P_{Li}-\sum_{j\neq i}{P_{ij}}-P_{ie}
  \end{equation}
  \begin{equation}\label{eq_EX_model_b}
    Q_{i}=-Q_{Li}-\sum_{j\neq i}{Q_{ij}}-Q_{ie}+b_{i}v_i^2,
  \end{equation}
\end{subequations}
\end{shrinkeq}
where $P_i/Q_i$, $P_{Li}/Q_{Li}$, $P_{ij}/Q_{ij}$, and $P_{ie}/Q_{ie}$ denote the active/reactive branch flow of the border bus $i$ to the internal network, the equivalent active/reactive load of border bus $i$, the active/reactive branch flow from bus $i$ to $j$, and the active/reactive branch flow from bus $e_i$ to $i$, respectively.
By substituting the branch flow equations \eqref{eq_branch_flow} into \eqref{eq_EX_model}, the active/reactive branch flow $P_i/Q_i$ can be formulated as:
\begin{equation}\label{eq_EX_lr}
  \begin{aligned}
  h_i^{ex}\!(v,\theta)&\!=\!\!
  \sum_{j\neq i}{\!\left[
    \alpha_{j1}^{ex}(
      v_i^2\!-\!v_iv_j\cos\theta_{ij}
    )
    \!+\!
    \alpha_{j2}^{ex}v_iv_j\sin\theta_{ij}
    \right]}&\\
    &+
    \alpha_3^{ex}v_i^2+ \alpha_4^{ex}v_i+\alpha_5^{ex}.&
  \end{aligned}
\end{equation}
Similar to the case of branch flow linearization, the parameters $\bm{\alpha^{ex}}=\left[\alpha_{j1}^{ex},\alpha_{j2}^{ex},\alpha_3^{ex},\alpha_4^{ex},\alpha_5^{ex}\right]$ can be obtained from historical operational data by OLS algorithm.

The values of $\bm{\alpha^{ex}}$ are related to the power system and its operational characteristics.
Thus, one can implement physical knowledge to constrain the value of $\bm{\alpha^{ex}}$, see ~\cite{yu2017optimal} for details.
\begin{equation}\label{eq_physical_ex}
  \mathbf{\underline{B}^{ex}}
  \leq
  \bm{\alpha^{ex}}
  \leq
  \mathbf{\overline{B}^{ex}}.
\end{equation}

\begin{figure}[hbt!]
	\centering
		\includegraphics[width=.65\linewidth]{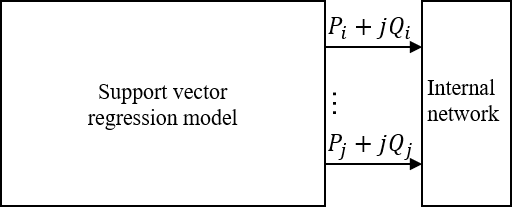}
	\caption{The external network equivalence model using support vector regression.}
	\label{fig_external_model_svr}
\end{figure}

The above data-driven models are linear regression models.
To compare cases under different data-driven models, we also implement the support vector regression (SVR) method to learn the external network model.
SVR has shown high accuracy in learning power flow models with an appropriate setting of inputs~\cite{yu2017robust}.
The SVR model does not have any assumed physical circuits for the external network, as shown in \figurename~\ref{fig_external_model_svr}.
Hence, the aforementioned physical knowledge~\eqref{eq_physical_ex} can only be applied in linear regression method rather than SVR method.
For border bus $i$, we set the input vector as $\bm{x_i}=\left[v_i,v_j,\sin\theta_{ij},\cos\theta_{ij}\right]$, where $j$ represents all other border buses.
Given $m$ snapshots $\bm{x_i^1} \ldots \bm{x_i^m}$, the SVR model can learn the active/reactive branch flow from the following formulation:
\begin{equation}\label{eq_SVR}
  h_i^{svr}(\bm{x_i})
  =\sum_{i=k}^{m}\left(\alpha_{i}^{+}-\alpha_{i}^{-}\right)
  \mathcal{K}\left( \bm{x_i^k}, \bm{x_i}\right)
  +b,
\end{equation}
where $\alpha_{i}^{+}/\alpha_{i}^{-}$, $b$, and $\mathcal{K}(\cdot)$ are the dual variables in the SVR optimization process, the constant term, and the kernel function, respectively.
Note that for simplicity, we only introduce the dual form of the SVR problem.
The non-linearity of the SVR comes from the kernel function $\mathcal{K}(\cdot)$.
Considering the formulation of power flow equations, the second order polynomial kernel has sufficient richness to learn the power flow model:
\begin{equation}\label{eq_kernel}
  \mathcal{K}\left( \bm{x_i^k}, \bm{x_i}\right)
  =\left(
    \bm{x_i^k}\bm{x_i^T}+c
    \right)^2,
\end{equation}
where $c$ is a constant term.
The parameters $\alpha_{i}^{+}/\alpha_{i}^{-}$ have two characteristics:
1) The values are constrained considering the optimization problem of SVR:
\begin{equation}\label{eq_SVR_C}
  0
  \leq
  \alpha_{i}^{+}
  \leq
  C
  \text{ }
  0
  \leq
  \alpha_{i}^{-}
  \leq
  C,
\end{equation}
where the constant term $C$ denotes the weight of loss in the primal optimization problem.
2) $\alpha_{i}^{+}/\alpha_{i}^{-}$ are sparse. 
$\alpha_{i}^{+}$ is non-zero only when the difference of the real value and the prediction value are larger than a threshold $\epsilon^{svr}$: $y_i-h_i^{svr}(\bm{x_i})>\epsilon^{svr}$.
Similarly, $\alpha_{i}^{+}$ is non-zero only when $h_i^{svr}(\bm{x_i})-y_i>\epsilon^{svr}$.
We adopt the big M method to satisfy the aforementioned two characteristics:
\begin{subequations}\label{eq_SVR_constrain}
  \begin{equation}\label{eq_SVR_constrain_a}
    M(u_i^+-1)\leq d_i^+-\epsilon^{svr}\leq Mu_i^+
    \text{  }
    0\leq \alpha_{i}^{+} \leq Cu_i^+
  \end{equation}
  \begin{equation}\label{eq_SVR_constrain_b}
    M(u_i^--1)\leq d_i^--\epsilon^{svr}\leq Mu_i^-
    \text{  }
    0\leq \alpha_{i}^{-} \leq Cu_i^-.
  \end{equation}
\end{subequations}
\subsection{Experimental Results}
The experimental setup is the same as in Section~\ref{section_branch_flow}.
Hyper parameters only in this Section are set as: $C=0.2$, $\epsilon^{svr}=0.01$, and $l=2$.
We use the IEEE 39-bus system in this case study, with external buses: \#1-\#2, \#25-\#30, and \#37-\#39; border buses: \#3, \#9, and \#17; and internal buses: \#4-\#8, \#10-\#16, \#18-\#24, and \#31-\#36.
We generate 500 snapshots of training data and test the accuracy on 500 newly generated snapshots using the least squares method.
The errors of different methods on different border buses are listed in Table \ref{table_ex_total}.

\begin{table}[ht]
	\centering
	\renewcommand{\arraystretch}{1.1}
	\caption{The Mean Absolute Errors of Different Border Buses.}
	\label{table_ex_total}
		\begin{tabular}{@{}lllllll@{}}
      \toprule
         & $P_i$, \#3 & $P_i$, \#9 & $P_i$, \#17 & $Q_i$, \#3 & $Q_i$, \#9 & $Q_i$, \#17 \\ \midrule
      Errors & 0.0408           & 0.0283          & \textbf{0.0534}           & 0.0047           & 0.0352           & 0.0070           \\
      \bottomrule
      \end{tabular}
\end{table}

As shown in Table \ref{table_ex_total}, the active branch flow of border bus \#17 has the largest testing error.
Hence, we evaluate the generalization bound of the active branch flow of border bus \#17.

Three different methods are compared:
1) \textbf{LR}: The linear regression model without any physical knowledge. Set \eqref{eq_EX_lr} as the hypothesis in \eqref{eq_opt_slow};
2) \textbf{LRBox}: The linear regression model with physical knowledge. Set \eqref{eq_EX_lr} as the hypothesis in \eqref{eq_opt_slow}, with the maximum and minimum parameter bound constraints \eqref{eq_physical_ex} as the physical constraints in \eqref{eq_opt_slowe};
3) \textbf{SVR}: The SVR model. Set \eqref{eq_SVR} as the hypothesis in \eqref{eq_opt_slow}, with constraints \eqref{eq_SVR_constrain} added to the problem.

\begin{figure}[hbt!]
	\centering
		\includegraphics[width=.85\linewidth]{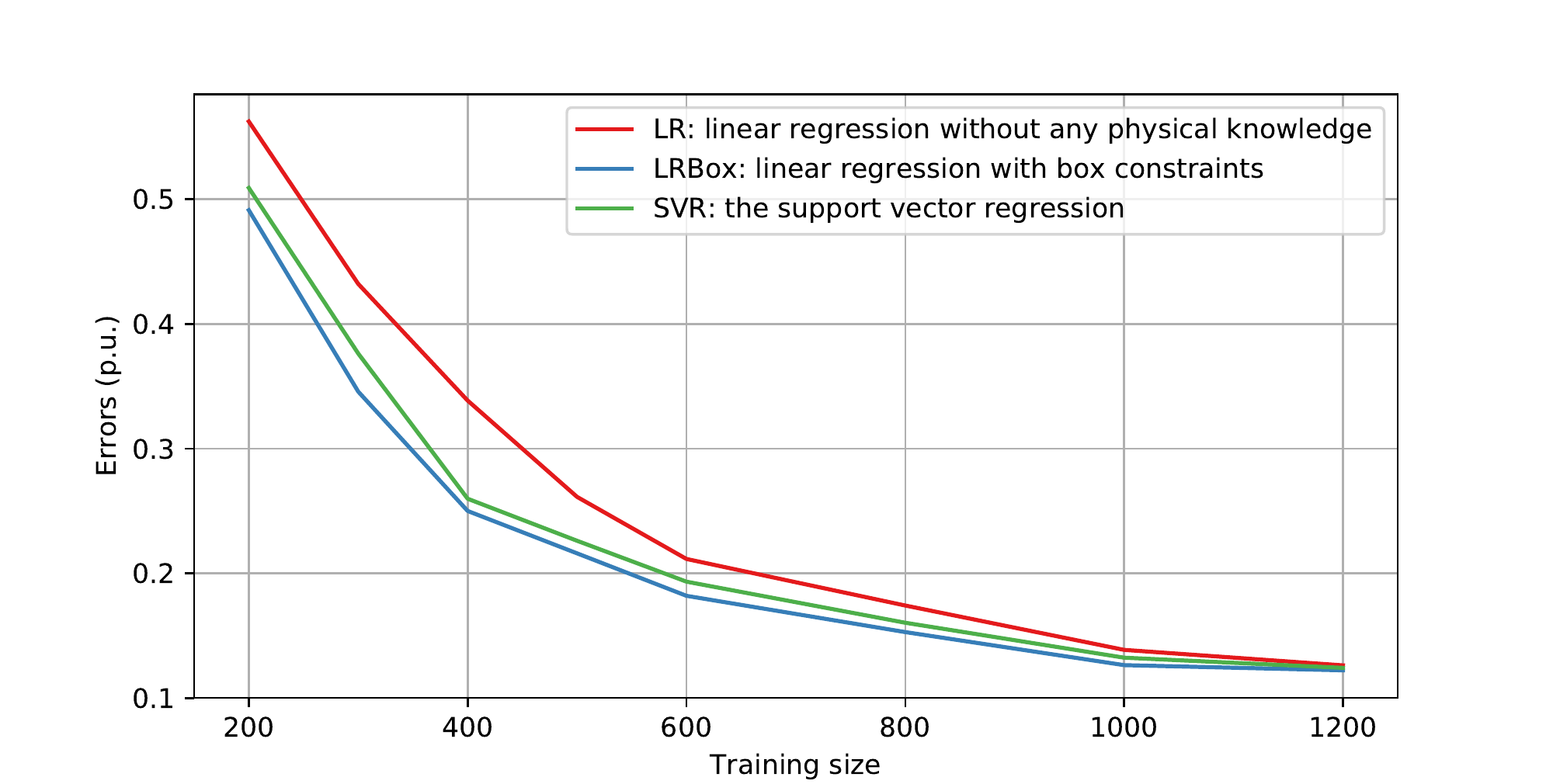}
	\caption{
	Comparisons of different models and physical knowledge in the case of active branch flow of Branch \#17.
	The total generalization bounds of \textbf{LR}, \textbf{LRBox}, and \textbf{SVR} under different amount of training data are compared.
	}
	\label{fig_physical_ex}
\end{figure}

We bound generalization errors of methods \textbf{LR}, \textbf{LRBox}, and \textbf{SVR} under different amount of training data, as shown in \figurename\ref{fig_physical_ex}.
An impressive result is that the generalization bound of the linear regression model \textbf{LR} is larger than that of the nonlinear \textbf{SVR} model.
It is because without the physical knowledge, the parameters of \textbf{LR} are not constrained.
In contrast, the parameters of \textbf{SVR} are constrained by \eqref{eq_SVR_constrain}.
From \textbf{Definition~\ref{def_empirical_rademacher}}, model \textbf{LR} has more richness to fit random noise than \textbf{SVR}, and thus has larger Rademacher complexity.
When the physical knowledge is added to the linear regression model, however, the generalization bound is significantly reduced.
The bound of \textbf{LRBox} is the lowest among all the methods.
We can conclude some suggestions that can help the implementation of the data-driven methods in practice.
The physical knowledge has significant effect of improving the generalization performance when training data size $m\leq 1000$.
Besides, it is suggested that when physical knowledge of linear regression method is not available (e.g., due to information barriers), \textbf{SVR} has better generalization performance; and when physical knowledge of linear regression method is available, \textbf{LRBox} has better generalization performance.

\begin{figure}[hbt!]
	\centering
		\includegraphics[width=1\linewidth]{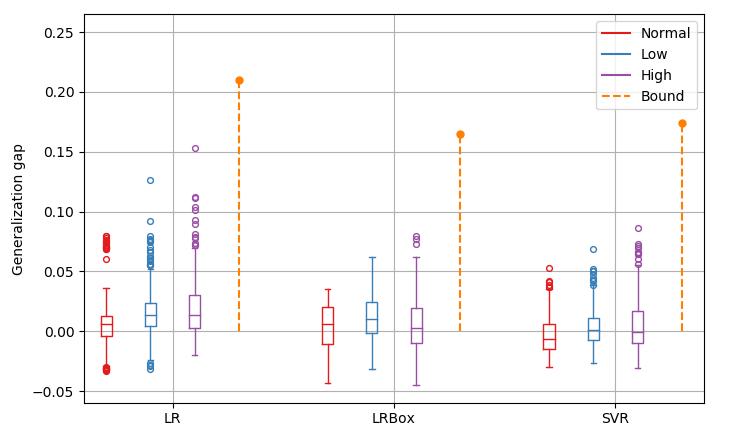}
	\caption{
	The boxplot of generalization gap of method \textbf{LR}, \textbf{LRBox}, and \textbf{SVR}. 
	All the models are trained with dataset \textbf{Normal} and are tested with dataset \textbf{Normal}, \textbf{Low}, and \textbf{High}.
	20 training datasets are generated with each dataset contains 500 samples.
	20 testing datasets are generated to evaluate the generalization performance.
	The bound value computed by~\eqref{eq_empir_MRC_re} is shown in the dashed line.
	}
	\label{fig_testing_ex}
\end{figure}

Then, we use some generated testing scenarios to evaluate the bound under different models and different data distributions.
In details, we demonstrate the distribution of generalization gap $L\left(h\right)-\hat{L}_{\bm{x}}(h)$ for model \textbf{LR}, \textbf{LRBox}, and \textbf{SVR}.
We also show the result where the training datasets and the testing datasets are from different distributions.
Three types of datasets are generated: 1) \textbf{Low}: Data are generated under the load range of [0.85-0.95]; 2) \textbf{Normal}: Data are generated under the load range of [0.95-1.05]; 3) \textbf{High}: Data are generated under the load range of [1.05-1.15].
We train the models with dataset \textbf{Normal} and test the model with all the datasets.
The generalization gap testing results and the bound value are shown in~\figurename~\ref{fig_testing_ex}.
From the testing result, the method with larger theoretical bound has a larger generalization gap in testing scenarios.
This provides evidence that the theoretical MRC bound can shape the influence of different models and physical knowledge on generalization performance.
In \figurename~\ref{fig_testing_ex}, the theoretical bound are much larger than the empirical testing results (the \textbf{Normal} boxplot).
One reason is that the bound implements some inequalities that is not tight enough, such as the bounded variance inequality in \textbf{Theorem~\ref{lemma_bounded_variance}} and the Jensen's inequality in~\eqref{eq_MRC_1}.
Future work will further improve the tightness by exploring more concentration properties.
Regarding the data distribution, the proposed MRC bound assumes that the training and the testing data are drawn from the same underlying distribution, so that there is no theoretical guarantee the proposed method is valid under different training and testing distributions.
From the testing results, all the models that use different testing distributions have larger generalization gap than using the same distribution.
In the above case, however, our bound is conservative and can also bound the generalization gap results from different datasets.
Furthermore, the theoretical bound can also evaluate the relative performance of different models and physical knowledge when the distributions of testing datasets and training datasets are different.
\section{Conclusion}
\label{section_conclusion}


Our work provides theoretical insights into the error bounds of data-driven models in power grid analysis, accounting for the influence of different data-driven models, physical knowledge, and the amount of training data.
We extend the existing theory and propose a new MRC bound by adopting a new concentration inequality. 
Our proposed method is theoretically tighter on regression problems.
Furthermore, the evaluation of the bound is formulated as a mathematical programming problem, which can consider the effects of physical knowledge in power grid analysis.
We conduct two case studies: branch flow linearization and external network equivalent, to demonstrate how our MRC bound can be used to evaluate the performance of different data-driven models.
Our work rethinks the design and validation of data-driven models towards more wide-spread power industry applications.
Further research will focus on expanding the applications of the proposed MRC bound method, improving the tightness of the bound, and exploring provable guarantees on unsupervised and reinforcement learning power system applications.



\begin{appendices}
\section{Proof of \textbf{Theorem \ref{theo_empirical_MRC}}}
\subsection{Proof of inequality \eqref{eq_MRC}}
\label{appen_proof_MRC}
For simplicity, we denote $\mathcal{H}(m)$ as $\mathcal{H}$ in the Appendix.
We firstly restate a theorem from Section 6.4 of ~\cite{boucheron2013concentration}.
\begin{theorem}[Bounded variance inequality]
  \label{lemma_bounded_variance}
  Given function $\phi :\mathcal{X}^m\rightarrow \mathbb{R}$, let:
  \begin{equation}\label{eq_difference}
    c_{i}=\left|\phi\left(x_{1}, \ldots, x_{i}, \ldots, x_{m}\right)-\phi\left(x_{1}, \ldots, x_{i}^{\prime}, \ldots, x_{m}\right)\right|.
  \end{equation}
  If $\phi$ satisfies:
    $\sum_{i=1}^{m} c_{i}^{2} \leq v$
  , then for $\forall t\leq 0$:
  \begin{equation}
    \operatorname{Pr}\left[\phi\left(x_{1}, \ldots, x_{m}\right) \!-\! \mathrm{E}\left[\phi\left(x_{1}, \ldots, x_{m}\right)\right]\geq t\right] \!\leq\! e^{-t^{2}/2v}.
  \end{equation}
\end{theorem}
The bounded variance inequality develops the concentration property of function $\phi$ under a bounded variance assumption.
It is a mild assumption compared with the McDiarmid's inequality~\cite{boucheron2013concentration} used in traditional Rademacher complexity bound.
Then we introduce the following Lemma to discover the relationship between $L\left(h\right)$ and $\hat{L}_{\bm{x}}\left(h\right)$:
\begin{lemma}\label{cor_phi}
  Set $\phi$ as the supreme delta value of the generalization error and the empirical error: 
  \begin{shrinkeq}{-0.5ex}
  \begin{subequations}\label{eq_function_phi}
  \begin{flalign}\label{eq_function_phi_a}
    &
    \phi\left(x_{1}, \ldots, x_{m}\right)=
    \sup _{h\in\mathcal{H}}\left(L\left(h\right)-\hat{L}_{\bm{x}}\left(h\right)\right)
    &
  \end{flalign}
  \begin{flalign}\label{eq_function_phi_b}
    &
    \qquad\qquad
    =
    \sup _{h\in\mathcal{H}}\left(L\left(h\right)\!-\!\frac{1}{m}\sum_{i=1}^{m}{\left|f(x_i)\!-\!h(x_i)\right|}\right)
    &
  \end{flalign}
  \end{subequations}
  \end{shrinkeq}
  Then, under the assumption in \eqref{eq_condition_MRC}, for $\forall \delta \in (0,1)$, with probability at least $1-\delta$, the following holds:
  \begin{equation}\label{eq_cor1}
    \phi\left(x_{1}, \ldots, x_{m}\right) \!\leq\! \underset{\bm{x}\sim\mathcal{D}}{\mathrm{E}}\left[\phi\left(x_{1}, \ldots, x_{m}\right)\right]
    \!+\!e\sqrt{(2\log\frac{2}{\delta })/m}.
  \end{equation}
\end{lemma}
\begin{proof}
  From the definition of \eqref{eq_difference} and \eqref{eq_function_phi_b}, we have:
  \begin{equation}\label{eq_cor11}
    c_i=(1/m)\left|f\left(x_i\right)-h\left(x_i\right)\right|.
  \end{equation}
  From assumption \eqref{eq_condition_MRC}, we have:
  \begin{equation}\label{eq_cor12}
    \sum_{i=1}^{m} c_{i}^{2}=\frac{1}{m^2}\sum_{i=1}^{m}{\left[ f\left( {{x}_{i}} \right)-h\left( {{x}_{i}} \right) \right]^2} \leq \frac{e^2}{m}=v.
  \end{equation}
  Apply \textbf{Theorem~\ref{lemma_bounded_variance}} with $v=e^2/m$ and $t=e\sqrt{2\log(2/\delta)/m}$ and finish the proof.
\end{proof}

Then, from the definition of \eqref{eq_function_phi}, we have:
  \begin{shrinkeq}{-1ex}
  \begin{subequations}\label{eq_MRC_1}
    \begin{flalign}\label{eq_MRC_11}
      &{\underset{{\bm{x}\sim\mathcal{D}}}{\mathrm{E}}}\!\left[\phi\left(x_{1}, \ldots, x_{m}\right)\right]\!=\!\!\!
      {\underset{{\bm{x}\sim\mathcal{D}}}{\mathrm{E}}}\!\left[\sup _{h\in \mathcal{H}}\left(L(h)-\hat{L}_{\bm{x}}(h)\right)\right]&
    \end{flalign}
    \begin{flalign}\label{eq_MRC_12}
      &={\underset{{\bm{x}\sim\mathcal{D}}}{\mathrm{E}}}\left[\sup _{h\in \mathcal{H}}
      \left({\underset{{\bm{x}'\sim\mathcal{D}}}{\mathrm{E}}}\left(\hat{L}_{\bm{x}'}\left(h\right)\right)-\hat{L}_{\bm{x}}(h)\right)\right]&
    \end{flalign}
    \begin{flalign}\label{eq_MRC_13}
      &={\underset{{\bm{x}\sim\mathcal{D}}}{\mathrm{E}}}\left[\sup _{h\in \mathcal{H}}
      \left({\underset{{\bm{x}'\sim\mathcal{D}}}{\mathrm{E}}}\left(\hat{L}_{\bm{x}'}\left(h\right)-\hat{L}_{\bm{x}}(h)\right)\right)\right]&
    \end{flalign}
    \begin{flalign}\label{eq_MRC_14}
      &\leq{\underset{{\bm{x},\bm{x}'\sim\mathcal{D}}}{\mathrm{E}}}\left[\sup _{h\in \mathcal{H}}
      \left(\hat{L}_{\bm{x}'}\left(h\right)-\hat{L}_{\bm{x}}(h)\right)\right],&
    \end{flalign}
  \end{subequations}
\end{shrinkeq}
where \eqref{eq_MRC_12} denotes the fact that generalization error is the expectation of empirical error over $\bm{x}'\sim\mathcal{D}$;
\eqref{eq_MRC_13} uses the fact that sample $\bm{x}$ and $\bm{x}'$ are independent;
and \eqref{eq_MRC_14} holds by the Jensen's inequality~\cite{needham1993visual}, using the convexity of supremum function: $\sup[\mathrm{E}(x)]\leq \mathrm{E}[\sup(x)]$.

We build sample set $\bm{\hat{x}}$ and $\bm{\hat{x}}'$ from $\bm{x}$ and $\bm{x}'$, by randomly swapping samples between $\bm{x}$ and $\bm{x}'$ with the probability of 0.5.
Denote $x_i\in \bm{x}$ and $x'_i\in \bm{x}'$.
Recall the definition of Rademacher variables $\boldsymbol{\sigma}$ in \textbf{Definition~\ref{def_empirical_rademacher}}, we have:
\begin{equation}\label{eq_MRC_rade_vari}
  \hat{L}_{\bm{\hat{x}}'}\left(h\right)-\hat{L}_{\bm{\hat{x}}}(h)
  =\frac{1}{m}\sum_{i=1}^{m}\sigma_i\left(l(h,x'_i)-l(h,x_i)\right).
\end{equation}
Since $\bm{x}$, $\bm{x}'$, $\bm{\hat{x}}$, and $\bm{\hat{x}}'$ are from same distribution, we have:
\begin{shrinkeq}{-0.9ex}
  \begin{subequations}\label{eq_MRC_2}
    \begin{flalign}\label{eq_MRC_21}
      &{\underset{{\bm{x},\bm{x}'\sim\mathcal{D}}}{\mathrm{E}}}\!\!\left[\sup _{h\in \mathcal{H}}\!\!
      \left(\hat{L}_{\bm{x}'}\!\!\left(h\right)\!\!-\!\!\hat{L}_{\bm{x}}(h)\right)\right]
      \!\!=\!\!\!\!\!{\underset{{\bm{\hat{x}},\bm{\hat{x}}'\sim\mathcal{D}}}{\mathrm{E}}}\!\!\left[\sup _{h\in \mathcal{H}}\!\!
      \left(\!\hat{L}_{\bm{\hat{x}}'}\!\!\left(h\right)\!\!-\!\!\hat{L}_{\bm{\hat{x}}}(h)\!\right)\!\right]&
    \end{flalign}
    \begin{flalign}\label{eq_MRC_22}
      &={\underset{{\bm{x},\bm{x}'\sim\mathcal{D},\boldsymbol{\sigma}}}{\mathrm{E}}}
      \left[\frac{1}{m}\sup _{h\in \mathcal{H}}\left(\sum_{i=1}^{m}{\sigma}_i\left(l(h,x'_i)-l(h,x_i)\right)\right)\right]&
    \end{flalign}
    \begin{flalign}\label{eq_MRC_23}
      &\leq\!\!\!\!\!\!\!
      {\underset{{\bm{x}'\sim\mathcal{D},\boldsymbol{\sigma}}}{\mathrm{E}}}\!\!
      \left[\frac{1}{m}\sup _{h\in \mathcal{H}} 
      \sum_{i=1}^{m}\!\!{\sigma}_{i}l\left(h,x'_i\right)\!\right]
      \!\!+\!\!\!\!\!\!
      {\underset{{\bm{x}\sim\mathcal{D},\boldsymbol{\sigma}}}{\mathrm{E}}}\!\!
      \left[\frac{1}{m}\!\!\sup _{h\in \mathcal{H}} 
      \sum_{i=1}^{m}\!\!\!\!-{\sigma}_{i}l\left(h,x_i\right)\!\right]
      &
    \end{flalign}
    \begin{flalign}\label{eq_MRC_24}
      &
      =2{\underset{{\bm{x}\sim\mathcal{D},\boldsymbol{\sigma}}}{\mathrm{E}}}\!\!
      \left[\frac{1}{m}\sup _{h\in \mathcal{H}} 
      \sum_{i=1}^{m}\!\!{\sigma}_{i}l\left(h,x_i\right)\right]
      =2\mathfrak{R}(\mathcal{H}),
      &
    \end{flalign}
  \end{subequations}
\end{shrinkeq}
where \eqref{eq_MRC_22} comes from \eqref{eq_MRC_rade_vari};
\eqref{eq_MRC_23} holds by the sub-additivity of supremum function $\sup (x+y)\leq \sup(x)+\sup(y)$;
\eqref{eq_MRC_24} holds by \textbf{Definition~\ref{def_empirical_rademacher}}.
Finally, the proof is finished by combining \textbf{Lemma~\ref{cor_phi}} and \eqref{eq_MRC_1}-\eqref{eq_MRC_2}:
\begin{shrinkeq}{-0.5ex}
  \begin{subequations}\label{eq_MRC_3}
    \begin{flalign}\label{eq_MRC_31}
      &
      L\left(h\right) - \hat{L}_{\bm{x}}\left(h\right)
      \leq
      \sup _{h\in\mathcal{H}}\left(L\left(h\right)-\hat{L}_{\bm{x}}\left(h\right)\right)
      &
    \end{flalign}
    \begin{flalign}\label{eq_MRC_32}
      &
      \qquad 
      \qquad 
      \qquad
      \leq
      2\mathfrak{R}(\mathcal{H})+e\sqrt{(2\log\frac{2}{\delta })/m}.
      &
    \end{flalign}    
  \end{subequations}
\end{shrinkeq}
Eq.~\eqref{eq_MRC_3} denotes the term in~\eqref{eq_MRC_32} can bound the maximum generalization gap for $h\in \mathcal{H}$.
This is the reason that \textbf{Theorem~\ref{theo_empirical_MRC}} is valid for all possible training and testing scenarios.
\subsection{Proof of inequality \eqref{eq_empir_MRC1}}
\label{appen_proof_relation}
Apply \textbf{Theorem~\ref{lemma_bounded_variance}} with $\phi\left(x_i\ldots x_m\right)=\widehat{\mathfrak{R}}(\mathcal{H})$.
Similar with \eqref{eq_cor11}-\eqref{eq_cor12}, we set $\sum_{i=1}^{m} c_{i}^{2} \leq v=e^2/m$.
Set $t=e\sqrt{2\log(2/\delta)/m}$ and finish the proof.

\section{Configuration of $k$ in \textbf{Step 3}}

The factor $k$ is used to update the bound of MSE by $e^{new}\leftarrow kL\left(h\right)$, where ${e^{new}}^2$ is the bound of MSE and $L(h)$ is the bound of MAE.
The $e^{new}$ is used to calculate the empirical Rademacher complexity $\widehat{\mathfrak{R}}'(\mathcal{H}(e_{comp}^{new})\cap \mathcal{P})$ and the sample uncertainty $3e_{samp}^{new}\sqrt{2\log(2/\delta)/m}$.
$k$ is different for the update of $e_{comp}^{new}$ and $e_{samp}^{new}$.

For the update of $e_{comp}^{new}$, $k=1$. 
Because although we set ${e_{comp}^{new}}^2$ as the bound of MSE, the empirical Rademacher complexity is calculated under the bound of MAE as we derive from \textbf{Proposition~\ref{theo_MAE}}.
Hence, the $L(h)$ of this iteration can be directly used as $e_{comp}^{new}=L(h)$ in the next iteration.

For the update of $e_{samp}^{new}$, $k$ is determined by the following inequality, indicating that the square root mean is smaller than the arithmetic mean times $k$.
\begin{equation}\label{eq_square_root_arithmetic}
    \left(
      \frac{1}{m}
      \sum_{i=1}^{m}{l(h,x_i)}^2
    \right)
    ^
    {1/2}
    \leq
    k
    \left(
    \frac{1}{m}
    \sum_{i=1}^{m}l(h,x_i)
    \right).
\end{equation}
In practice, $k$ can be obtained by assuming a certain distribution of the error and then sampling on this distribution.
For example, we assume the error $f(x_i)-h(x_i)$ follows a Gaussian distribution in the case study.
Then we sample the ratio of the square root mean to the arithmetic mean of $l(h,x_i)$ and find that the value of $k=1.4$ can bound the ratio under Gaussian assumption.
One can also apply different assumptions on the error distribution to get the value of $k$.
\end{appendices}

\ifCLASSOPTIONcaptionsoff
  \newpage
\fi


\bibliography{IEEEabrv,myReference}

\begin{thebibliography}{10}
\providecommand{\url}[1]{#1}
\csname url@samestyle\endcsname
\providecommand{\newblock}{\relax}
\providecommand{\bibinfo}[2]{#2}
\providecommand{\BIBentrySTDinterwordspacing}{\spaceskip=0pt\relax}
\providecommand{\BIBentryALTinterwordstretchfactor}{4}
\providecommand{\BIBentryALTinterwordspacing}{\spaceskip=\fontdimen2\font plus
\BIBentryALTinterwordstretchfactor\fontdimen3\font minus
  \fontdimen4\font\relax}
\providecommand{\BIBforeignlanguage}[2]{{%
\expandafter\ifx\csname l@#1\endcsname\relax
\typeout{** WARNING: IEEEtran.bst: No hyphenation pattern has been}%
\typeout{** loaded for the language `#1'. Using the pattern for}%
\typeout{** the default language instead.}%
\else
\language=\csname l@#1\endcsname
\fi
#2}}
\providecommand{\BIBdecl}{\relax}
\BIBdecl

\bibitem{liu2018data}
Y.~Liu, N.~Zhang, Y.~Wang, J.~Yang, and C.~Kang, ``Data-driven power flow
  linearization: A regression approach,'' \emph{IEEE Transactions on Smart
  Grid}, 2018.

\bibitem{yu2017optimal}
J.~Yu, W.~Dai, W.~Li, X.~Liu, and J.~Liu, ``Optimal reactive power flow of
  interconnected power system based on static equivalent method using border
  pmu measurements,'' \emph{IEEE Transactions on Power Systems}, vol.~33,
  no.~1, pp. 421--429, 2017.

\bibitem{jalali2019designing}
M.~Jalali, V.~Kekatos, N.~Gatsis, and D.~Deka, ``Designing reactive power
  control rules for smart inverters using support vector machines,'' \emph{IEEE
  Transactions on Smart Grid}, 2019.

\bibitem{cremer2019optimization}
J.~Cremer, I.~Konstantelos, and G.~Strbac, ``From optimization-based machine
  learning to interpretable security rules for operation,'' \emph{IEEE
  Transactions on Power Systems}, 2019.

\bibitem{thams2019efficient}
F.~Thams, A.~Venzke, R.~Eriksson, and S.~Chatzivasileiadis, ``Efficient
  database generation for data-driven security assessment of power systems,''
  \emph{IEEE Transactions on Power Systems}, 2019.

\bibitem{zheng2018svm}
X.~Zheng, X.~Geng, L.~Xie, D.~Duan, L.~Yang, and S.~Cui, ``A svm-based setting
  of protection relays in distribution systems,'' in \emph{2018 IEEE Texas
  Power and Energy Conference (TPEC)}.\hskip 1em plus 0.5em minus 0.4em\relax
  IEEE, 2018, pp. 1--6.

\bibitem{babaei2019data}
S.~Babaei, C.~Zhao, and L.~Fan, ``A data-driven model of virtual power plants
  in day-ahead unit commitment,'' \emph{IEEE Transactions on Power Systems},
  2019.

\bibitem{geng2019data}
X.~Geng and L.~Xie, ``Data-driven decision making in power systems with
  probabilistic guarantees: Theory and applications of chance-constrained
  optimization,'' \emph{Annual Reviews in Control}, 2019.

\bibitem{bui2019double}
V.-H. Bui, A.~Hussain, and H.-M. Kim, ``Double deep q-learning-based
  distributed operation of battery energy storage system considering
  uncertainties,'' \emph{IEEE Transactions on Smart Grid}, 2019.

\bibitem{carriere2019integrated}
T.~Carriere and G.~Kariniotakis, ``An integrated approach for value-oriented
  energy forecasting and data-driven decision-making application to renewable
  energy trading,'' \emph{IEEE Transactions on Smart Grid}, vol.~10, no.~6, pp.
  6933--6944, 2019.

\bibitem{karagiannopoulos2019data}
S.~Karagiannopoulos, P.~Aristidou, and G.~Hug, ``Data-driven local control
  design for active distribution grids using off-line optimal power flow and
  machine learning techniques,'' \emph{IEEE Transactions on Smart Grid}, 2019.

\bibitem{chen2019exploiting}
Y.~Chen, Y.~Tan, and B.~Zhang, ``Exploiting vulnerabilities of load forecasting
  through adversarial attacks,'' in \emph{Proceedings of the Tenth ACM
  International Conference on Future Energy Systems}, 2019, pp. 1--11.

\bibitem{bor2019adversarial}
M.~C. Bor, A.~K. Marnerides, A.~Molineux, S.~Wattam, and U.~Roedig,
  ``Adversarial machine learning in smart energy systems,'' in
  \emph{Proceedings of the Tenth ACM International Conference on Future Energy
  Systems}, 2019, pp. 413--415.

\bibitem{venzke2019verification}
A.~Venzke and S.~Chatzivasileiadis, ``Verification of neural network behaviour:
  Formal guarantees for power system applications,'' \emph{arXiv preprint
  arXiv:1910.01624}, 2019.

\bibitem{ashtiani2018nearly}
H.~Ashtiani, S.~Ben-David, N.~Harvey, C.~Liaw, A.~Mehrabian, and Y.~Plan,
  ``Nearly tight sample complexity bounds for learning mixtures of gaussians
  via sample compression schemes,'' in \emph{Advances in Neural Information
  Processing Systems}, 2018, pp. 3412--3421.

\bibitem{bhattarai2019big}
B.~P. Bhattarai, S.~Paudyal, Y.~Luo, M.~Mohanpurkar, K.~Cheung, R.~Tonkoski,
  R.~Hovsapian, K.~S. Myers, R.~Zhang, P.~Zhao \emph{et~al.}, ``Big data
  analytics in smart grids: state-of-the-art, challenges, opportunities, and
  future directions,'' \emph{IET Smart Grid}, 2019.

\bibitem{koltchinskii2001rademacher}
V.~Koltchinskii, ``Rademacher penalties and structural risk minimization,''
  \emph{IEEE Transactions on Information Theory}, vol.~47, no.~5, pp.
  1902--1914, 2001.

\bibitem{cortes2013learning}
C.~Cortes, M.~Kloft, and M.~Mohri, ``Learning kernels using local rademacher
  complexity,'' in \emph{Advances in neural information processing systems},
  2013, pp. 2760--2768.

\bibitem{oneto2019local}
L.~Oneto, S.~Ridella, and D.~Anguita, ``Local rademacher complexity machine,''
  \emph{Neurocomputing}, 2019.

\bibitem{maximov2018rademacher}
Y.~Maximov, M.-R. Amini, and Z.~Harchaoui, ``Rademacher complexity bounds for a
  penalized multi-class semi-supervised algorithm,'' \emph{Journal of
  Artificial Intelligence Research}, vol.~61, pp. 761--786, 2018.

\bibitem{mohri2018foundations}
M.~Mohri, A.~Rostamizadeh, and A.~Talwalkar, \emph{Foundations of machine
  learning}.\hskip 1em plus 0.5em minus 0.4em\relax MIT press, 2018.

\bibitem{boucheron2013concentration}
S.~Boucheron, G.~Lugosi, and P.~Massart, \emph{Concentration inequalities: A
  nonasymptotic theory of independence}.\hskip 1em plus 0.5em minus 0.4em\relax
  Oxford university press, 2013.

\bibitem{valiant1984theory}
L.~G. Valiant, ``A theory of the learnable,'' in \emph{Proceedings of the
  sixteenth annual ACM symposium on Theory of computing}.\hskip 1em plus 0.5em
  minus 0.4em\relax ACM, 1984, pp. 436--445.

\bibitem{vapnik1971uniform}
V.~Vapnik and A.~Y. Chervonenkis, ``On the uniform convergence of relative
  frequencies of events to their probabilities,'' \emph{Theory of Probability
  and its Applications}, vol.~16, no.~2, p. 264, 1971.

\bibitem{seeger2002pac}
M.~Seeger, ``Pac-bayesian generalisation error bounds for gaussian process
  classification,'' \emph{Journal of machine learning research}, vol.~3, no.
  Oct, pp. 233--269, 2002.

\bibitem{bartlett2005local}
P.~L. Bartlett, O.~Bousquet, S.~Mendelson \emph{et~al.}, ``Local rademacher
  complexities,'' \emph{The Annals of Statistics}, vol.~33, no.~4, pp.
  1497--1537, 2005.

\bibitem{germain2016pac}
P.~Germain, F.~Bach, A.~Lacoste, and S.~Lacoste-Julien, ``Pac-bayesian theory
  meets bayesian inference,'' in \emph{Advances in Neural Information
  Processing Systems}, 2016, pp. 1884--1892.

\bibitem{reeb2018learning}
D.~Reeb, A.~Doerr, S.~Gerwinn, and B.~Rakitsch, ``Learning gaussian processes
  by minimizing pac-bayesian generalization bounds,'' in \emph{Advances in
  Neural Information Processing Systems}, 2018, pp. 3337--3347.

\bibitem{holland2019pac}
M.~Holland, ``Pac-bayes under potentially heavy tails,'' in \emph{Advances in
  Neural Information Processing Systems}, 2019, pp. 2711--2720.

\bibitem{kuck2018approximate}
J.~Kuck, A.~Sabharwal, and S.~Ermon, ``Approximate inference via weighted
  rademacher complexity,'' in \emph{Thirty-Second AAAI Conference on Artificial
  Intelligence}, 2018.

\bibitem{boyd2004convex}
S.~Boyd and L.~Vandenberghe, \emph{Convex optimization}.\hskip 1em plus 0.5em
  minus 0.4em\relax Cambridge university press, 2004.

\bibitem{bullen2013means}
P.~S. Bullen, D.~S. Mitrinovic, and M.~Vasic, \emph{Means and their
  Inequalities}.\hskip 1em plus 0.5em minus 0.4em\relax Springer Science \&
  Business Media, 2013, vol.~31.

\bibitem{hu2019ensemble}
R.~Hu and Q.~Li, ``Ensemble learning based convexification of power flow with
  application in opf,'' \emph{arXiv preprint arXiv:1909.05748}, 2019.

\bibitem{donti2019matrix}
P.~L. Donti, Y.~Liu, A.~J. Schmitt, A.~Bernstein, R.~Yang, and Y.~Zhang,
  ``Matrix completion for low-observability voltage estimation,'' \emph{IEEE
  Transactions on Smart Grid}, 2019.

\bibitem{magnusson2020distributed}
S.~Magn{\'u}sson, G.~Qu, and N.~Li, ``Distributed optimal voltage control with
  asynchronous and delayed communication,'' \emph{IEEE Transactions on Smart
  Grid}, 2020.

\bibitem{li2018data}
X.~Li and K.~Hedman, ``Data driven linearized ac power flow model with
  regression analysis,'' \emph{arXiv preprint arXiv:1811.09727}, 2018.

\bibitem{yang2018general}
Z.~Yang, K.~Xie, J.~Yu, H.~Zhong, N.~Zhang, and Q.~Xia, ``A general formulation
  of linear power flow models: basic theory and error analysis,'' \emph{IEEE
  Transactions on Power Systems}, vol.~34, no.~2, pp. 1315--1324, 2018.

\bibitem{zimmerman2010matpower}
R.~D. Zimmerman, C.~E. Murillo-S{\'a}nchez, and R.~J. Thomas, ``Matpower:
  Steady-state operations, planning, and analysis tools for power systems
  research and education,'' \emph{IEEE Transactions on power systems}, vol.~26,
  no.~1, pp. 12--19, 2010.

\bibitem{yu2017robust}
J.~Yu, Y.~Weng, and R.~Rajagopal, ``Robust mapping rule estimation for power
  flow analysis in distribution grids,'' in \emph{2017 North American Power
  Symposium (NAPS)}.\hskip 1em plus 0.5em minus 0.4em\relax IEEE, 2017, pp.
  1--6.

\bibitem{needham1993visual}
T.~Needham, ``A visual explanation of jensen's inequality,'' \emph{The American
  mathematical monthly}, vol. 100, no.~8, pp. 768--771, 1993.

\end{thebibliography}

%


%
%
%




\end{document}